\documentclass[10pt,twocolumn,twoside]{IEEEtran} 
\usepackage{graphicx}
\usepackage{amsmath,amssymb,amsfonts}
\usepackage{nicefrac}
\usepackage{mathptmx}
\usepackage[latin1]{inputenc}
\usepackage{cite}
\usepackage[ruled,lined,linesnumbered]{algorithm2e}

\usepackage{rotating}
\newcommand{\vcell}[1]{\begin{centering}\begin{sideways}\begin{centering}{#1}\end{centering}\end{sideways}\end{centering}}

\def\bc{{\mathbb C}}
\def\bB{{\mathbb B}}
\def\bp{{\mathbf p}}
\def\bq{{\mathbf q}}
\def\br{{\mathbf r}}

\def\ba{{\mathbf a}}
\def\bb{{\mathbf b}}
\def\bx{{\bf x}}

\newcommand{\Proba}{\mathbb{P}}

\newtheorem{definition}{Definition}
\newtheorem{theorem}{Theorem}
\newtheorem{lemma}{Lemma}
\newtheorem{proposition}{Proposition}

\newenvironment{proof}[1][Proof]{\begin{trivlist}
\item[\hskip \labelsep {\bfseries #1}]}{\end{trivlist}}

\graphicspath{{./},{Images/}}

\author{Neus~Sabater$^*$,  Andr\'{e}s~Almansa$^{**}$ and~Jean-Michel~Morel$^*$\\
$^*$ENS Cachan, CNRS-CMLA. France.\\
$^{**}$Telecom ParisTech, CNRS-LTCI. France.
}

\begin{document}

\title{Meaningful Matches in Stereovision}

\maketitle

\begin{abstract}
This paper introduces a statistical method to decide whether two blocks in a pair of  images match reliably.
The method  ensures that the selected block matches are unlikely to have occurred ``just by chance.'' The new approach is based on the definition of a simple but faithful statistical
\textit{background model} for image blocks learned from the image itself. A theorem guarantees that under this model not more than a fixed number of wrong  matches occurs (on average) for the whole image. This fixed number (the number of false alarms) is the only method parameter.  Furthermore, the number of false alarms associated with each match measures its reliability.
This {\it a contrario} block-matching  method, however, cannot rule out false matches due to the presence of periodic objects in the images. But it is successfully  complemented by a parameterless \textit{self-similarity threshold}.
Experimental evidence shows that the proposed method also detects occlusions and incoherent motions due to vehicles and pedestrians in non simultaneous stereo.
\end{abstract}

\begin{IEEEkeywords}
Stereo vision, Block-matching, Number of False Alarms (NFA), \textbf{{\it a contrario}} detection.

\end{IEEEkeywords}

\section{Introduction}

Stereo  algorithms aim at reconstructing a 3D model from two or more images of the same scene acquired at different angles. This work only considers  previously stereo-rectified image pairs. In that case the 3D reconstruction  requires that the matched  points  in both images belong to the same horizontal epipolar line. The matching process of stereo image pairs has been studied in depth for more than four decades.  \cite{Brown03} and \cite{Scharstein02} contain a fairly complete comparison of the main methods. According to these surveys there are roughly two main classes of algorithms in binocular stereovision:  local matching methods and global methods.

Global methods aim at a coherent solution obtained by minimizing an energy functional containing matching fidelity terms and regularity constraints. The most efficient ones seem to be Belief Propagation \cite{Klaus06,Yang06}, Graph Cuts \cite{Kolmogorov05}, Dynamic Programming \cite{Ohta85,Forstmann04} and solvers of the multi-label problem \cite{Ishikawa03, Pock08}. They often resolve ambiguous matches by maintaining a coherence along the epipolar line (DP) or along and across epipolar lines (BP \& GC). They rely on a regularization term to eliminate outliers and reduce the noise. They give a match to all points which are not detected as occluded. Global methods are, however, at risk to make or propagate errors if the regularization term is not adapted to the scene. A classic example is when a large portion of the scene is nearly constant, for example a scene including a cloudless sky, since there is no information in such a region to compute reliable matches (see Fig. \ref{fig:flower_garden} for an example). On such ambiguous regions, global methods perform an interpolation by using the informative pixels. This interpolation can be lucky, as it is the case in most images of the Middlebury benchmark\footnote{http://vision.middlebury.edu/stereo/}. But it can also fail, as is apparent in the above example and in many outdoor scenes.  Furthermore, the energy in global methods, has at least two terms and one parameter weighting them (and sometimes three terms and two parameters \cite{Kolmogorov05}). These parameters are difficult to tune, and even to model. Thus, it remains a valid question how to rule out by a parameterless method the dubious regions where the matches cannot be scientifically demonstrated.

On the other hand local methods are simpler, but equally sensitive to local ambiguities. Local methods start by comparing features of the right and left images. These features can be blocks in block-matching methods, or even local descriptors \cite{Mikolajczyk03} like SIFT descriptors \cite{Lowe04,Rabin07}, curves \cite{Schmid00}, corners \cite{Harris88,Cao04}, etc.  The drawback of local methods is that they do not provide a dense map as global methods do (meaning that the percentage of matched points is lower than 100\%).

Recent years have therefore seen  a blooming of global methods, which reach the best performance in recent benchmarks such as the Middlebury dataset \cite{Scharstein02}. But our purpose is to show that local methods can also be competitive. This paper considers the common denominator of most local methods, block-matching. It shows that this tool is amenable to a local statistical decision rule telling us whether a match is reliable. In fact, not all the pixels in an image pair can be reliably matched in real scenes (40 to 80\% of pixels). The lack of corresponding points in the second image or the ambiguity in certain points stirs up gross errors in dense stereovision. In particular block-matching methods suffer from two mismatching causes that must be tackled one by one:

\begin{enumerate}
\item  The main mismatch cause in local methods is the absence of a theoretically well founded threshold to decide whether two blocks
really match or not. Our main goal here will be to define such a threshold by an {\it a contrario} block-matching (ACBM) rejection rule, ensuring that two blocks do not match ``just by chance.''

\item A second minor mismatch cause is the presence on the epipolar line of repetitive shapes or textures, a problem sometimes called ``stroboscopic phenomenon,'' or ``self-similarity.''  The proposed ACBM only rules out stochastic similarities,  not deterministic ones. While the ACBM rule mismatches repetitive patterns, these types of mismatches are easily eliminated by a simple self-similarity rule (SS). We shall, however, verify that a self-similarity rule by itself is far from reaching the ACBM performance. Both rules are necessary and complementary.
\end{enumerate}

The elimination of these two sorts of mismatches is a key issue in block-matching me\-thods. The problem of sifting out matching errors in stereovision has of course  been addressed many times. We shall discuss a choice of the significant contributions for each cause of mismatch.

{\it Occlusions} are still an open problem in stereovision and one of the main causes of mismatch. For this reason numerous stereo approaches focus on detecting them.  Global energy methods \cite{Kolmogorov05} address occlusions by adding a penalty term for occluded pixels in their energy function. In \cite{Szeliski01} the major contribution is the reasoning about visibility in multi-view stereo. \cite{Yang06} computes two disparity maps symmetrically and verifies the left-right coherence  to detect occluded pixels. \cite{Ohta85} asserts that if two points in the epipolar line match with two points with a different order then there is an occlusion. Again this can lead to errors if there are narrow objects in the scene. See also \cite{Egnal02}, which compares a choice of methods to detect occlusions.

Matching pixels in {\it poorly textured regions}, where noise dominates signal, is clearly the main cause of error. Based on local SNR estimates,  \cite{Delon07} has proposed to reject matches by thresholding the second derivative of the correlation function: the flatter the correlation function, the less reliable the match. In \cite{Sara02}, the mismatches due to weakly textured objects  or  to  {\it periodic structures} are considered. The author defines a confidently stable matching in order to establish the largest possible unambiguous matching at a given confidence level. Two parameters control the compromise between the percentage of bad matches and the match density of the map. Yet, the match density falls dramatically when the percentage of mismatches decreases. We will see that the method presented here is able to get denser disparity maps with less mismatches. Similarly, \cite{Manduchi99} tries to eliminate errors on repeated patterns. Yet their matches seem to concentrate mainly on image edges and therefore have a low density.
A more primitive version of the rejection method developed here was applied successfully to the detection of {\it moving and disappearing objects} in \cite{Sabater10}. This is a foremost problem in the quasi-simultaneous stereo usual in aerial or satellite imaging where vehicles and pedestrians perturb strongly the stereo matching process. The extended method  presented here deals with a much broader  class of mismatches, including those due to poor signal to noise ratio.

\subsection{Anterior Statistical \textit{A Contrario} Decision Methods}

 Because of the above mentioned reasons one cannot presuppose the existence of uniquely determined correspondences for all pixels in the image. Thus, a decision
must be taken on whether a block in the left image actually meaningfully matches or not its best match in the right
image. This problem will be addressed by the \textit{a contrario} approach  initiated by \cite{Desolneux07}. This method is generally viewed as an
adaptation to image analysis of classic hypothesis testing.  But it also has a psychophysical justification in the so-called Helmholtz principle, according to which all perceptions could be characterized as having a low probability of occurring in noise. Early versions of this principle  in computer vision are \cite{Lowe85}, \cite{Grimson91}, \cite{Stewart95}.

A probabilistic {\it a contrario } argument is also invoked in the SIFT method \cite{Lowe04}, which includes an empirical rejection threshold. A match between two descriptors $S_1$ and $S'_1$ is rejected if the second closest match $S'_2$ to $S_1$ is actually almost as close to $S_1$ as  $S'_2$ is. The typical distance ratio rejection threshold is $0.6$, which means that $S_2$ is accepted if $dist(S'_1,S_1)\leq 0.6\times dist(S'_2, S_1)$ and rejected otherwise. Interestingly, Lowe justifies this threshold by a probabilistic argument: if the second best match is almost as good as the first, this only means that both matches are likely to occur casually. Thus, they must be rejected.  Recently, \cite{Rabin07} proposed a rigorous theory for this intuitive method.  SIFT matches are accepted or rejected by an \textit{a contrario} methodology involving the Earth mover distance. The \textit{a contrario} methodology has also already been used in
stereo matching. \cite{Moisan04} proposed a probabilistic criterion to detect a rigid motion between two point sets taken from a stereo pair, and to estimate the fundamental matrix. This method, ORSA, shows improved robustness compared to a classic RANSAC method. In the context of foreground detection in video, \cite{Mittal04} proposed an \textit{a contrario} method for discriminating foreground from background pixels that was later refined by \cite{Patwardhan08}. Even though this problem has some points in common with stereo matching, it is in a way less strict, since it only needs to learn to discriminate two classes of pixels. Hence they do not need to resort to image blocks, but rely only on a 5 dimensional feature vector composed of the color and motion vector of each pixel.


Among influential related works, Robin {\it et al.} \cite{Robin09}  describe a method for change detection in a time series of Earth observation images.
The change region is defined as the complement of the maximal region where the time series does not change significantly. Thus, what is controlled by the {\it  a contrario } method is the number of false alarms (NFA) of the no-change region. This method can therefore  be regarded as an {\it a contrario} region matching method. It is fundamentally different from the method we shall present. Indeed, Robin's method assumes (in addition to the statistical background model) a statistical image model that the time series follows in the regions where no change occurs, which is not feasible in stereo matching.


The method in \cite{Nee08} is also worth mentioning.  It is an \textit{a contrario} method for detecting similar regions between two images. This method is a classic statistical test rather than an \textit{a contrario} detection method in the sense of \cite{Desolneux07}. Indeed, the role of the background model ($H_0$ hypothesis) and the structure to be tested ($H_1$ hypothesis) are reversed: This method only controls the false negative rate and not the false positive rate (as in typical \textit{a contrario} methods). Furthermore the significance level of the statistical test is set to $\alpha \approx 0.1$ in accordance with classical statistical testing, whereas as demonstrated in \cite{Desolneux07} the significance level can be made much more secure, of the order of $10^{-6}$.

 The {\it a contrario} model for region matching in stereo vision used in \cite{Igual07} is simple and efficient. The gradient orientations at all region pixels are assumed  independent and uniformly distributed in the background model. A more elaborate version learns the probability distribution of gradient orientation differences under the hypothesis that the disparity (or motion) is zero, and uses this distribution as a background model.  Still, pixels are all considered as independent under the background model. Once this background model is learned, a given disparity (or motion model) is considered as meaningful if the number of aligned gradient orientations is sufficiently large within the tested region. This region matching method works well, but requires an initial over-segmentation of the gray-level image which is later refined by an \textit{a contrario} region merging procedure. Because of the rough background model, false positive region matches can be observed.

The key to a good background or {\it a contrario} model in block-matching would be to learn
 a realistic probability distribution of the high-dimensional space of  image patches.
The seminal works  \cite{Muse06} and \cite{Cao08} in the context of shape matching (where shapes are represented as pieces of level lines of a fixed size) showed that high-dimensional shape distributions can be efficiently approximated by the tensor product of (well chosen) marginal distributions. The marginal laws are one-dimensional, and therefore easily learned. In \cite{Muse03} these marginals are learned along the orientations of the principle components. The present work can be viewed as an extension of this curve matching method to block-matching.

\cite{Burrus09} proposed an alternative way of choosing detection thresholds such that the number of false detections under a given background model is ensured to stay below a given threshold.
The procedure does not require analytical computations or decomposing the probability as a tensor product of marginal distributions.
Instead, detection thresholds are learned by Monte-Carlo simulations in a way that ensures the target NFA rate.
This method, that was developed in the context of image segmentation, involves the definition of a set of thresholds to determine whether two neighboring regions are similar. However, as in \cite{Nee08}, the detected event whose false positive rate is controlled is \textit{``the two regions are different,''} and not the one we are interested in in the case of region matching, namely \textit{``the two regions are similar.''}

In conclusion, the {\it a contrario} methodology is expanding to many matching decision rules, but does not seem  to have been previously applied to the block-matching problem. We shall now proceed to describe the {\it a contrario} or background model for block-matching. The proposed model is the simplest that worked, but the reader may wonder if a still simpler model could actually work. In the next section we analyze a list of simpler proposals, and we explain why they must be discarded.

\subsection{Choosing an Adequate \textit{A Contrario} Model for Patch Comparison.}
The goal of this section is is to reject simpler alternatives to the probabilistic block
model that will be used.  In recent years, patch models and patch spaces are becoming increasingly  popular. We refer
to \cite{Mairal08} and references therein for algorithms generating sparse bases of patch spaces. Here, our goal can
be formulated in one single question,  that clearly depends on the observed set of patches in one particular image
and not on the probability space of {\it all} patches. The question  is: \newline``{\it What is the probability that
given two images and two similar patches in these images, this similarity arises just by chance?}''\newline  The ``just
by chance'' implies the existence of a stochastic {\it background model}, often called  the {\it a contrario} model.

When trying to define a well suited model for image blocks, many possibilities open up. Simple arguments show, however,
that over-simplified models do not work. Let $H$ be the gray-level histogram of the second image $I'$. The simplest
{\it a contrario model} of all might simply assume that the observed values $I'(\bx)$ are instances of i.i.d. random
variables $\mathcal{I}'(\bx)$ with cumulative distribution $H$. This would lead us to affirm that pixels $\bq$ in image $I$
and $\bq'$ in image $I'$ are a meaningful match if their gray level difference is unlikely small,
    $$ \Proba [ |I(\bq) - \mathcal{I}'(\bq')| \leq |I(\bq) - I'(\bq')| := \theta ] \leq \frac{1}{N_{tests}}. $$
As we shall see later, the number of tests $N_{tests}$ is quite large in this case ($N_{tests} \approx 10^7$ for typical
 image sizes), since it must consider all possible pairs of pixels $(\bq,\bq')$ that may match. But such
 a small probability can be achieved (assume that $H$ is uniform over $[0,255]$) only if the threshold
 $\theta = |I(\bq) - I'(\bq')| < 128 \cdot 10^{-7}$.
On the other hand, $|I(\bq) - I'(\bq')|$ cannot be expected to be very small because both images are corrupted
 by noise, among other distortions. Even in a very optimistic setting, where there would be only a small noise distortion
 between both images (of about 1 gray level standard deviation), such a small difference would only happen for
  about a tiny proportion ($3.2*10^{-5}$) of the correct matches.

This means that a pixel-wise comparison would require an extremely strict detection threshold to ensure the absence of
false matches, but this leads to an extremely sparse detection (about thirty meaningful matches per mega-pixel image).
This suggests that the use of local information around the pixel is unavoidable.

The next simplest approach could be to compare blocks of a certain size $\sqrt{s} \times \sqrt{s}$ with the $\ell^2$ norm, and  with the same background model as
before. Thus, we could declare blocks $B_{\bq}$ and $B_{\bq'}$ as meaningfully similar if
       \begin{multline}
 \Proba \left[ \frac{1}{|B_0|} \sum_{\bx\in B_0} |I(\bq+\bx) - \mathcal{I}'(\bq'+\bx)|^2 \leq \right. \\
\left. \frac{1}{|B_0|} \sum_{\bx\in B_0} |I(\bq+\bx) - I'(\bq'+\bx)|^2 := \theta \right]  \leq \frac{1}{N_{tests}}
       \end{multline}

where $B_0$ is the block of size $\sqrt{s} \times \sqrt{s}$ centered at the position (0,0).
Now the test would  be passed for a more reasonable threshold ($\theta = 6, 28, 47$ for blocks of size $3 \times 3$,
$5\times 5$, $7\times 7$ respectively), which would ensure a much denser response. However, this {\it a contrario
} model is by far too naive and produces  many false matches. Indeed, blocks stemming from natural images are much more
regular than the white noise generated by the background model. Considering all pixels in a block as independent leads
to overestimating the similarity probability of two observed similar blocks. It therefore leads to an over-detection.

In order to fix this problem, we need a background model better reflecting the statistics of natural image blocks.
But directly learning such a probability distribution from a single image in dimension 81 (for $9\times 9$ blocks)  is
hopeless.

Fortunately, as pointed out in \cite{Muse06},  high-dimensional distributions of shapes can be approximated by the
tensor product of their adequately chosen marginal distributions. Such marginal laws, being one-dimensional, are
easily  learned from a single image. Ideally,  ICA (Independent Component Analysis) should be used to learn which marginal laws are the most
independent, but the simpler PCA analysis will show accurate enough for our purposes. Indeed,  it ensures that the
principal components are decorrelated, a first approximation to independence. Fig. \ref{patches} gives a visual
assessment of how well a local PCA model simulates image patches in a class. Nevertheless, the independence  assumption will be used as  a tool for building the a-contrario model. This independence is not an empirical finding on the set of patches.

\subsection{Plan}
Section \ref{NeighborhoodComparison}  introduces the  stochastic block model learned from a reference image. Section \ref{TheAContrarioModel}  presents the {\it a contrario} method applied to disparity estimation in stereo pairs and treats the main problem of deciding whether two pixels match. Theorem \ref{Laseuleproposition} is the main result of this section, ensuring a controlled number of false detections. Section \ref{sec:autosimilarity-threshold} tackles the stroboscopic   problem  by a parameterless method, and demonstrates the necessity and complementarity of the {\it a contrario} and self-similarity  rejections.
Experimental results and comparison with other methods are in Section
\ref{sec:experimental_results}.
  Section \ref{Conclusions} is  conclusive. An appendix  summarizes the algorithm and gives its complete pseudo-code.

\section{The {\it a contrario} Model for Block-Matching}\label{NeighborhoodComparison}

We shall denote by $\bq\!\!=\!\!(q_1,q_2)$ a pixel in the reference image $I$ and by $B_\bq$ a block centered at $\bq$. To fix ideas, the block will be a square throughout this paper, but this is by no means a restriction. A different shape (rectangle, disk) would be possible, and even a variable shape. Given a point $\bq$ and its block $B_{\bq}$ in the reference image, block-matching algorithms
look for a point $\bq'$ in the second image $I'$ whose block $B_{\bq'}$ is similar to $B_{\bq}$.

\subsection{Principal Component Analysis} \label{PCA}

\begin{figure}
\begin{minipage}{.45\columnwidth}
\centering
 \includegraphics[width=3.5cm]{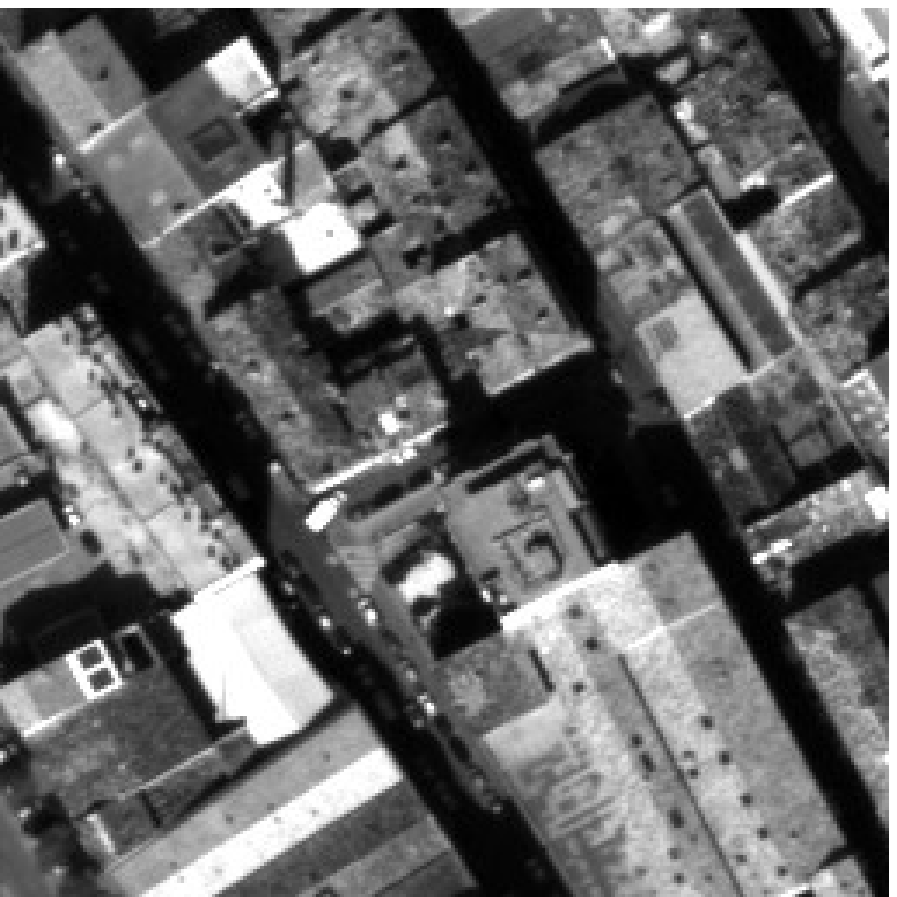}
\end{minipage}
\begin{minipage}{.45\columnwidth}
\centering
\includegraphics[width=1cm]{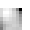}
\includegraphics[width=1cm]{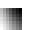}
\includegraphics[width=1cm]{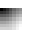}\\
\vspace{.2cm}
\includegraphics[width=1cm]{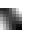}
\includegraphics[width=1cm]{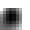}
\includegraphics[width=1cm]{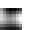}\\
\vspace{.2cm}
\includegraphics[width=1cm]{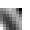}
\includegraphics[width=1cm]{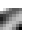}
\includegraphics[width=1cm]{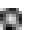}\\
\end{minipage}
\hfil
\caption{Left: Reference image of a stereo pair of images. Right: the nine first principal components of the  7$\times$7 blocks.  }
\label{pcs_partition}
\end{figure}

For building a simple {\it a contrario} model the principal component analysis can play a crucial role, as shown in \cite{Muse03}. Indeed, it allows for effective dimension reduction and decorrelates these dimensions, giving a first approximation to independence. This facilitates the construction of a probabilistic density function for the blocks as a tensor product of its marginal densities.
Let $B_{\bq}$ be the block of a pixel $\bq$ in the reference image and $(x_{1}^{\bq}, \ldots ,x_{s}^{\bq})$ the intensity gray levels in
$B_{\bq}$, where $s$ is the number of pixels in $B_{\bq}$. Let $n$ be the number of pixels in the image. Consider the
matrix $X=(x_{i}^{j})$ $1 \leq i \leq s ,\; 1 \leq j \leq n$ consisting of the set of all data vectors, one column per
pixel in the image. Then, the covariance matrix of the block is $C= \mathbb{E} (X-\bar{\bx}\textbf{1})(X-\bar{\bx}\textbf{1})^{T}$, where $\bar{x}$ is the column vector of size $s \times 1$ storing the mean values of matrix $X$ and $\textbf{1}=(1, \cdots, 1)$ a row vector of size $1 \times n$. Notice that $\bar{\bx}$ corresponds to the block whose $k$-th pixel is the average of all $k$-th  pixels of all blocks in the image. Thus, $\bar{\bx}$ is very close to a constant block, with the constant equal to the image average.
The eigenvectors of the covariance matrix are called principal components and are ortho\-gonal. They give the new coordinate system we shall use for blocks. Fig. \ref{pcs_partition} shows the first principal blocks.

Usually, the eigenvectors are sorted in order of decreasing eigenvalue. In that way the first principal components are
the ones that contribute most to the variance of the data set. By keeping the first $N<s$ components with larger
eigenvalues, the dimension is reduced but the significant information retained. While this global ordering could be
used to select the main components, a local ordering for each block will instead be used for the statistical matching rule. In other words, for each block, a new order for the principal components will be established given by the corresponding ordered PCA coordinates (the decreasing order is for the absolute values). In that way, comparisons of these components will be made from the most meaningful to the least meaningful one for this particular block.

Each block is represented by $N$ ordered coefficients $(c_{\sigma_{\bq}(1)}(\bq),\ldots,c_{\sigma_{\bq}(N)}(\bq))$, where $c_{i}(\bq)$ is the resulting coefficient after projecting $B_{\bq}$ onto the principal component $i \in \{ 1,\ldots,s \}$ and $\sigma_{\bq}$ the permutation representing the final order when ordering the absolute values of components for this particular $\bq$ in
decreasing order. By a slight abuse of notation we will write $c_{i}(\bq)$ instead of $c_{\sigma_{\bq}(i)}(\bq)$ knowing that it represents the local order of the best principal components. But notice that  $\sigma_{\bq}(1)=1$ for most $\bq$ because of the dominance of the first principal component. Moreover notice that this first component has a quite different coefficient histogram than the other ones (see Fig. \ref{pca_coordinate_histograms}), because it approximately computes a mean value of the block. Indeed, the barycenter of all blocks is roughly a constant block whose average grey value is  the
image average grey level. The set of blocks is elongated in the direction of the average grey level and, therefore, the
first component computes roughly an average grey level of the block. This explains why the  first component histogram
is similar to the image histogram.

\begin{figure}
\begin{center}
\begin{minipage}{.45\columnwidth}
\centering
\includegraphics[scale=2.5, angle=0,clip]{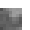}
\includegraphics[scale=2.5, angle=0,clip]{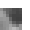}
\includegraphics[scale=2.5, angle=0,clip]{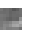}
\includegraphics[scale=2.5, angle=0,clip]{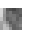}\\
\vspace{.1cm}
\includegraphics[scale=2.5, angle=0,clip]{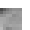}
\includegraphics[scale=2.5, angle=0,clip]{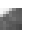}
\includegraphics[scale=2.5, angle=0,clip]{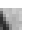}
\includegraphics[scale=2.5, angle=0,clip]{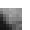}\\
\vspace{.1cm}
\includegraphics[scale=2.5, angle=0,clip]{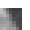}
\includegraphics[scale=2.5, angle=0,clip]{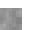}
\includegraphics[scale=2.5, angle=0,clip]{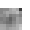}
\includegraphics[scale=2.5, angle=0,clip]{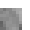}\\
\vspace{.1cm}
\includegraphics[scale=2.5, angle=0,clip]{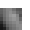}
\includegraphics[scale=2.5, angle=0,clip]{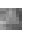}
\includegraphics[scale=2.5, angle=0,clip]{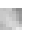}
\includegraphics[scale=2.5, angle=0,clip]{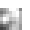}\\
(a)\\
\end{minipage}
 \begin{minipage}{.45\columnwidth}
\centering
\includegraphics[scale=2.5, angle=0,clip]{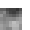}
\includegraphics[scale=2.5, angle=0,clip]{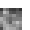}
\includegraphics[scale=2.5, angle=0,clip]{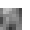}
\includegraphics[scale=2.5, angle=0,clip]{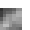}\\
\vspace{.1cm}
\includegraphics[scale=2.5, angle=0,clip]{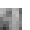}
\includegraphics[scale=2.5, angle=0,clip]{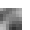}
\includegraphics[scale=2.5, angle=0,clip]{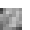}
\includegraphics[scale=2.5, angle=0,clip]{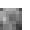}\\
\vspace{.1cm}
\includegraphics[scale=2.5, angle=0,clip]{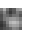}
\includegraphics[scale=2.5, angle=0,clip]{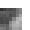}
\includegraphics[scale=2.5, angle=0,clip]{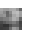}
\includegraphics[scale=2.5, angle=0,clip]{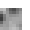}\\
\vspace{.1cm}
\includegraphics[scale=2.5, angle=0,clip]{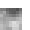}
\includegraphics[scale=2.5, angle=0,clip]{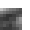}
\includegraphics[scale=2.5, angle=0,clip]{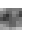}
\includegraphics[scale=2.5, angle=0,clip]{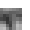}\\
(b)\\
\end{minipage}
\caption{(a) Patches of the reference image, chosen at random. (b) Simulated random blocks
following the law of the reference image. This experiment illustrates the (relative) adequacy of the {\it a contrario } model. Nevertheless, the PCA components are empirically uncorrelated, but of course not independent.    }
\label{patches}
\end{center}
\end{figure}
\subsection{{\it A Contrario} Similarity Measure between Blocks} \label{TheAContrarioModel}

\begin{definition}[{\it A contrario} model] \label{defacontrariomodel}
We call {\it a contrario block model} associated with a reference image  a  random block $\bB$  described by its (random) components  $\bB =(\bc_1, \dots, \bc_s)$  on the PCA basis of the blocks of the reference image, satisfying
 \begin{itemize}\item the  components $\bc_i$, $i=1, \dots, s$
  are independent random variables;
\item for each $i$, the law of $\bc_i$ is the empirical histogram  of the $i$-th PCA component $c_i(\cdot)$ of the blocks of the reference image.
\end{itemize}
\end{definition} The reference image will be the secondary image $I'$. Fig. \ref{patches} shows patches generated according to the above {\it a contrario}  block model and compares them to blocks picked at random in the reference image.
    The {\it a contrario} model will be  used for computing a  block resemblance probability
 as the product of the marginal resemblance probabilities of the $\bc_i$
  in the {\it a contrario} model, which is justified by the independence of $\bc_i$ and $\bc_j$ for $i\neq j$. There is a
  strong adequacy of the {\it a contrario} model to the empirical model, since
  the PCA transform  ensures that $\bc_i$ and $\bc_j$ are uncorrelated
  for $i\neq j$, a first approximation of the independence requirement.

We start by defining the resemblance probability between two blocks for a single component.
Denote by $H_{i}(\cdot):=H_{i}(c_{i}(\cdot))$   the normalized cumulative histogram of the $i$-th PCA block component $c_{i}(\cdot)$ for the secondary image $I'$.

\begin{definition}[Resemblance probability]\label{defempiricalprobability}
Let $B_{\bq}$ be a block in $I$ and $B_{\bq'}$ a block in $I'$. Define the probability that a random block $\bB$ of $I'$ resembles
$B_{\bq}$  as closely as $B_{\bq'}$ does in the $i$-th component by

\begin{equation*}
\widehat{p^{i}}_{\bq\, \bq'}=
\begin{cases}
H_{i}(\bq')                      & \text{if $H_{i}(\bq')-H_{i}(\bq) > H_{i}(\bq)  $;} \\
1-H_{i}(\bq')                    & \text{if $H_{i}(\bq)-H_{i}(\bq') > 1-H_{i}(\bq)$} \\
2|H_{i}(\bq)-H_{i}(\bq')| & \text{otherwise.}
\end{cases}
\end{equation*}
\end{definition}
Fig. \ref{cumulative_histo} illustrates how the resemblance probability $\widehat{p^{i}}_{\bq\, \bq'}$ is computed and Fig. \ref{pca_coordinate_histograms} shows  empirical marginal densities.

\begin{figure}[h]
\begin{center}
\centering
\includegraphics[scale=0.4]{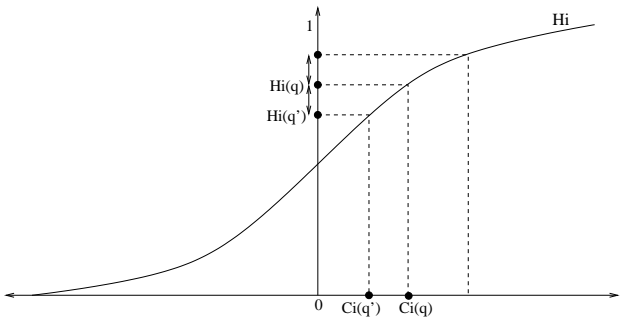}
\end{center}
\caption{Normalized cumulative histogram of $i$-th PCA coordinates of the secondary image. $c_i(\bq)$ is the $i$-th PCA coordinate value in the first image. The  resemblance probability $\widehat{p^{i}}_{\bq\, \bq'}$ for the $i$-th component is twice the distance $|H_{i}(\bq)-H_{i}(\bq')|$ when $H_{i}(\bq)$ is not too close to the values $0$ or $1$.}
\label{cumulative_histo}
\end{figure}

\begin{figure*}
\centering
\includegraphics[width=0.16\textwidth]{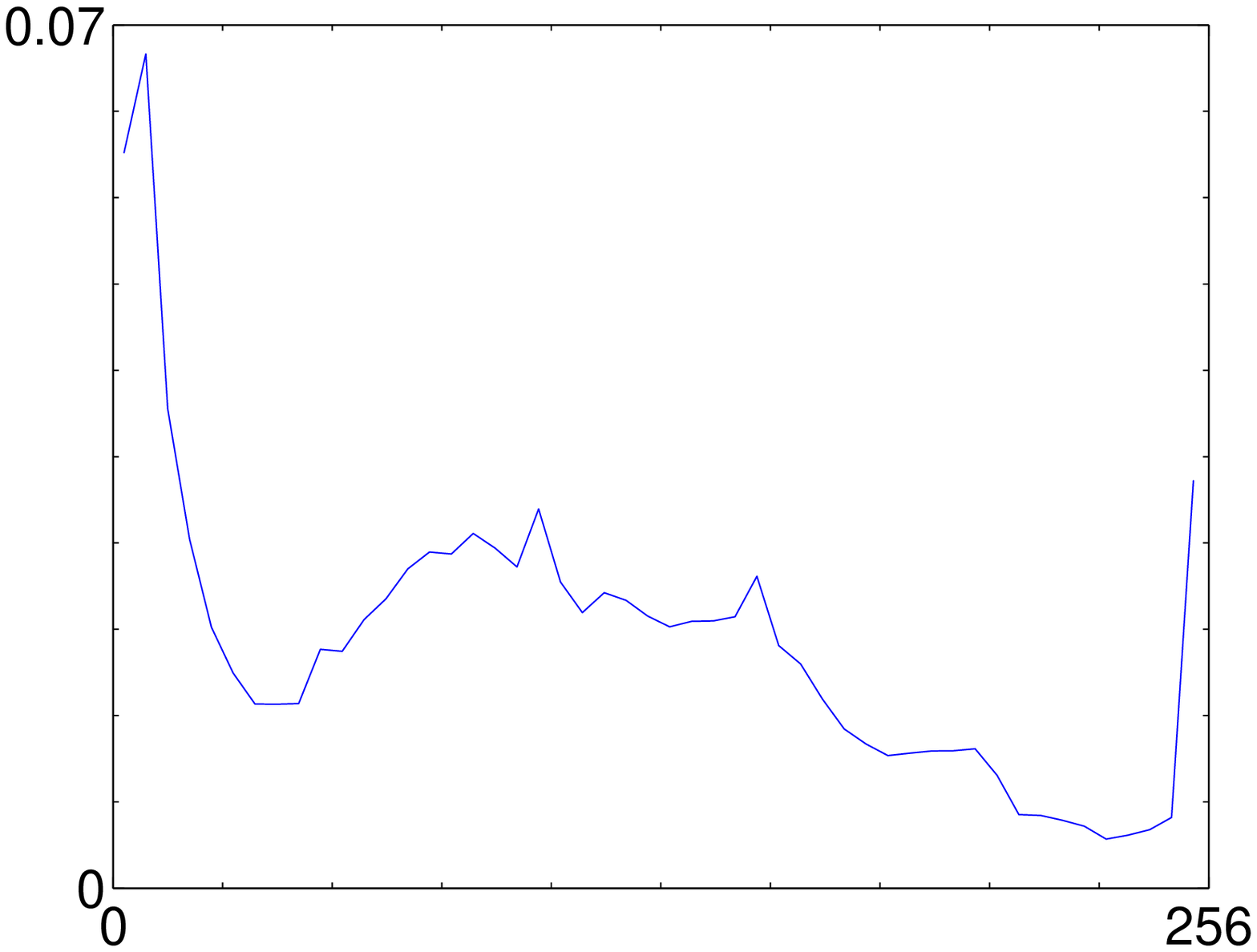}
 \includegraphics[width=0.16\textwidth]{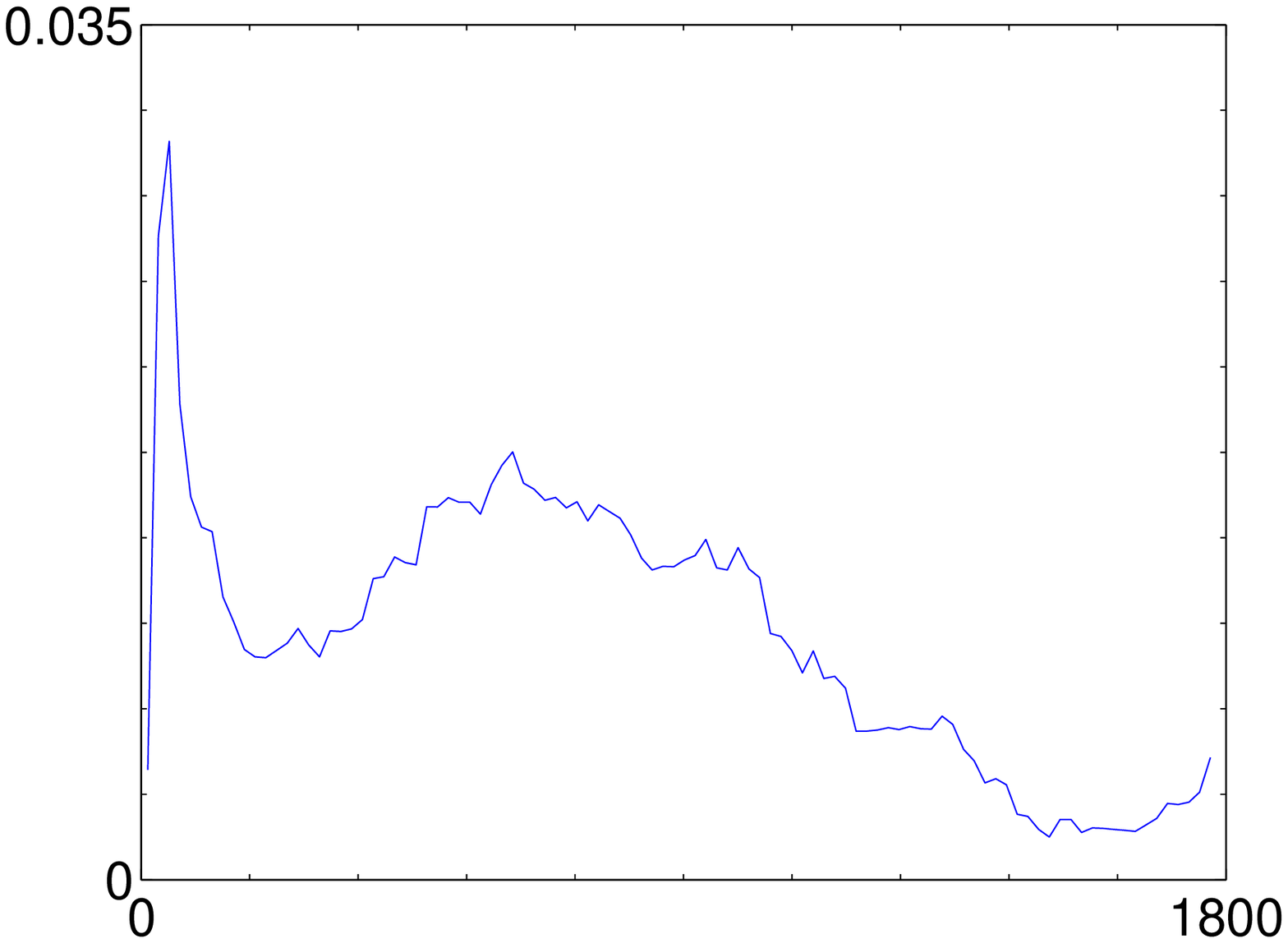}
 \includegraphics[width=0.16\textwidth]{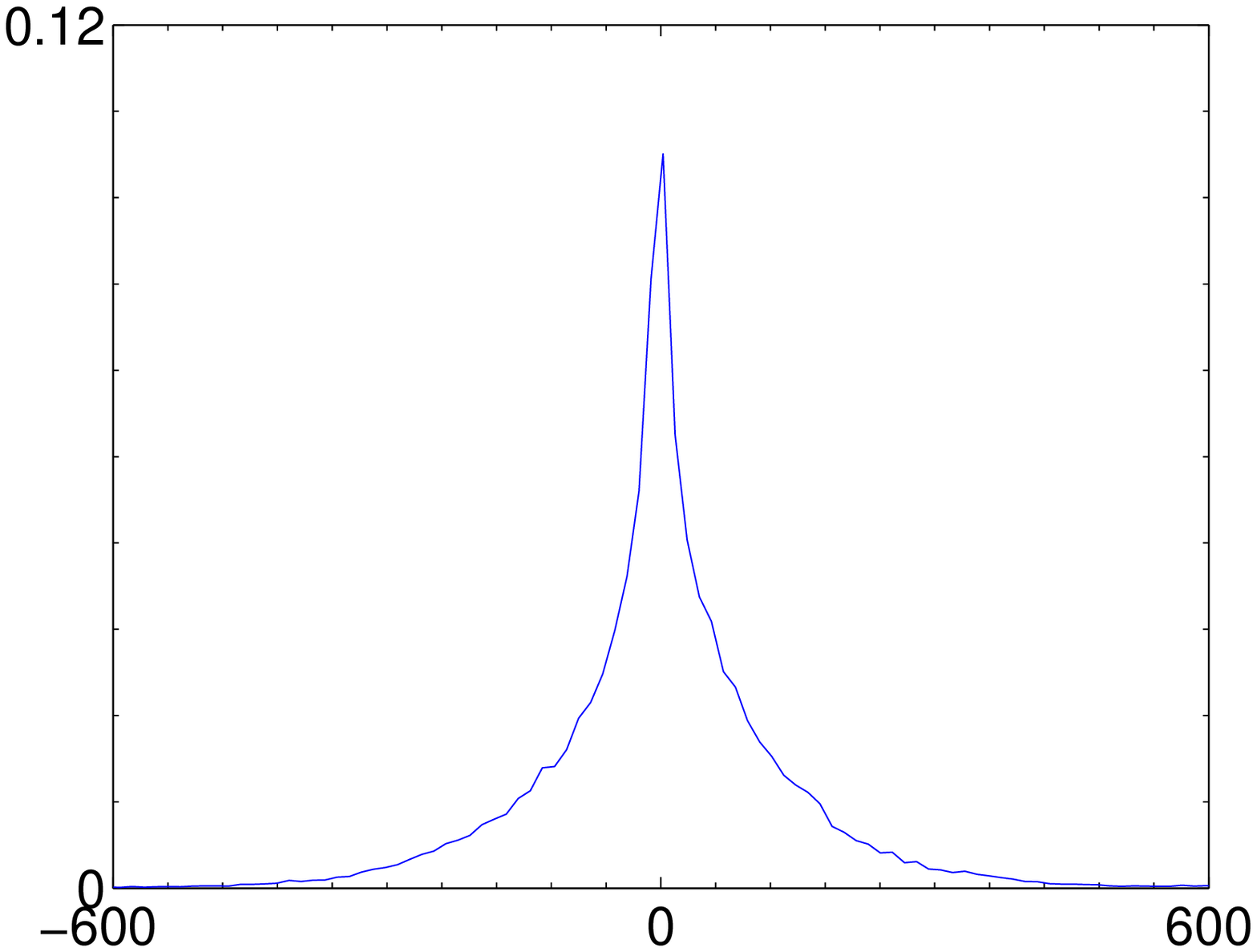}
 \includegraphics[width=0.16\textwidth]{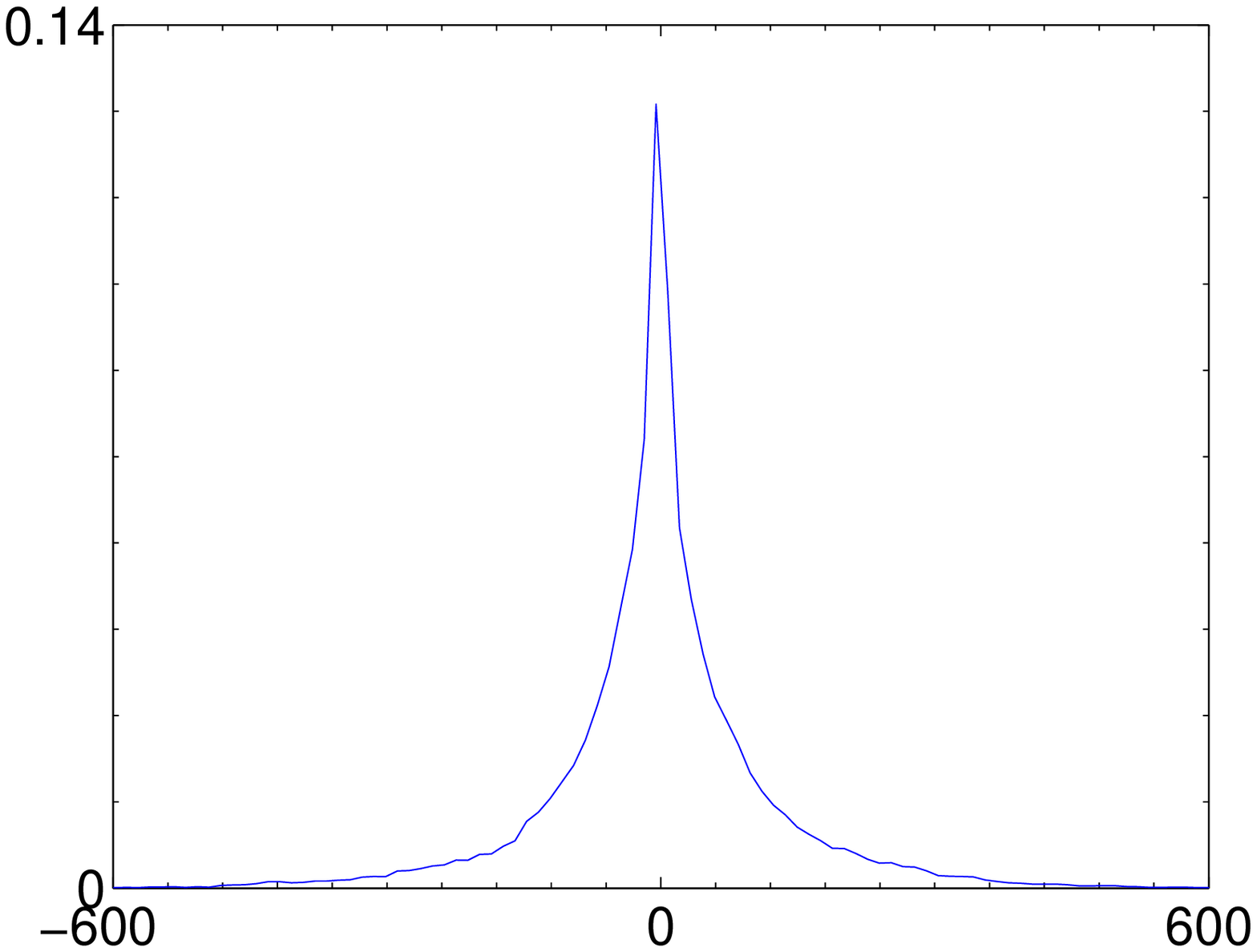}
 \includegraphics[width=0.16\textwidth]{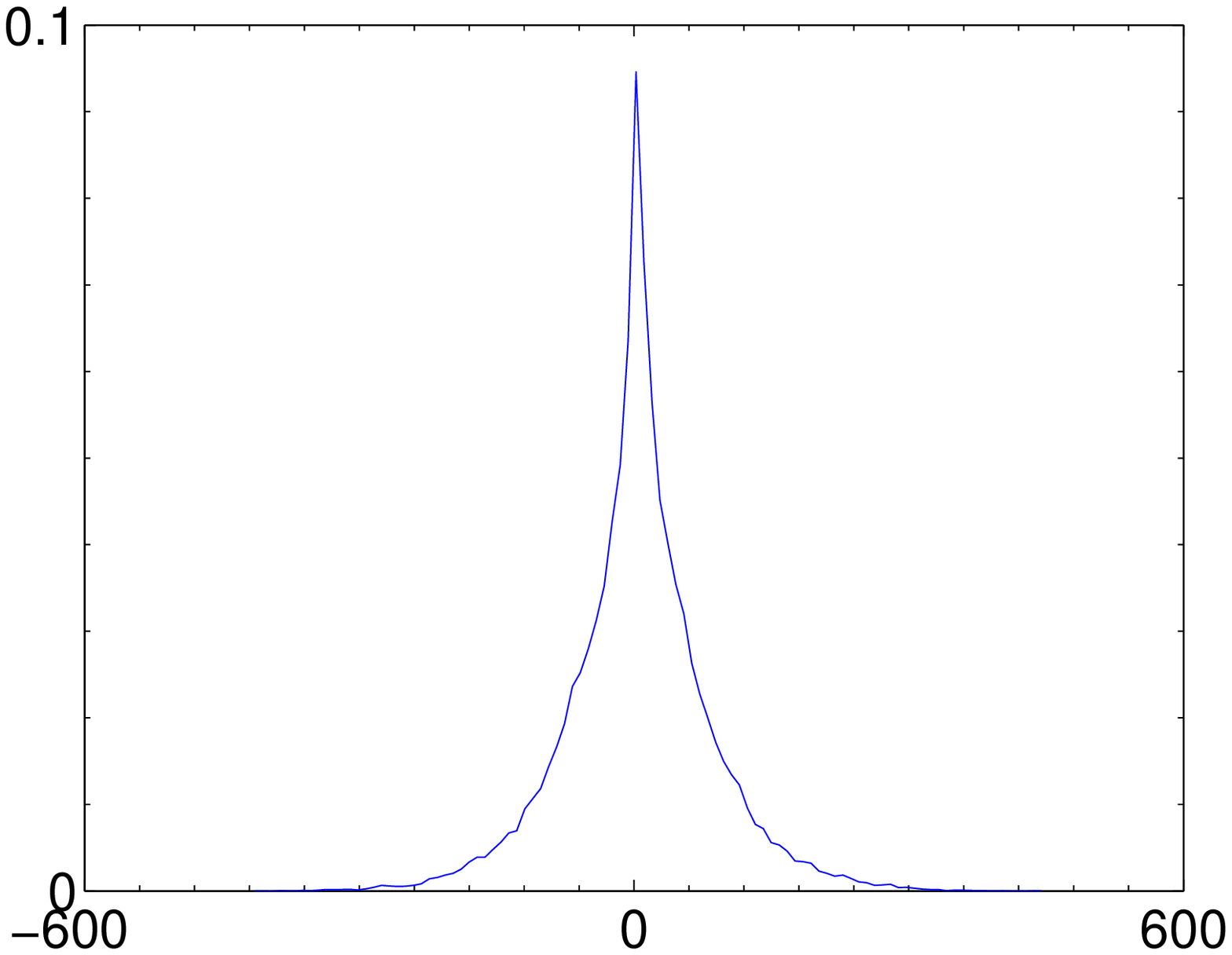}
 \includegraphics[width=0.16\textwidth]{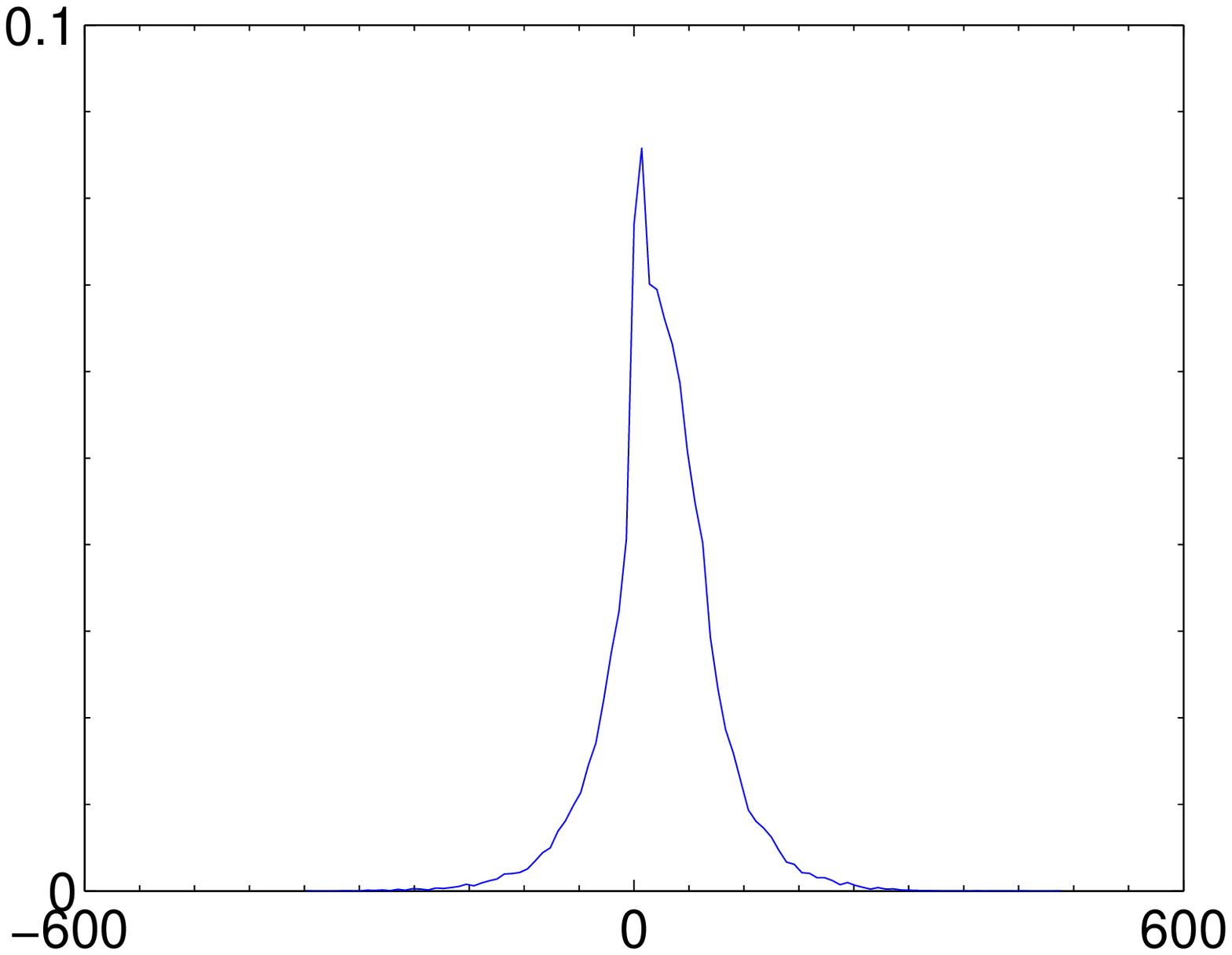}

\caption{Histogram of the reference image, followed by the  first five histograms of the block PCA coordinates.  The first principal component roughly computes a mean of the block, which explains why its histogram  is so similar to the image histogram.}
\label{pca_coordinate_histograms}
\end{figure*}

\subsection{Robust Similarity Distance}

The first principal components of $B_\bq$, being in decreasing order, contain the relevant information on the block. Thus, if two blocks are not similar
for one of the first components, they should not be matched, even if their next components are similar. Due to this fact, the
components of  $B_\bq$ and another block $B_{\bq'}$ must be compared with a non-decreasing exigency level. In addition, in the {\it a contrario} model, the number of tested correspondences should be as small as possible to reduce the number of false alarms. A quantization  of the tested resemblance probabilities is therefore required to limit the number of tests.

These  two remarks lead to define the quantized resemblance probability as the smallest non-decreasing sequence of quantized probabilities bounding from above the sequence $\widehat{ \, p^{i}}_{\bq\, \bq'}$.

\begin{definition}[Quantized probability]\label{def:quantized-non-decreasing-probability}
Let $B_{\bq}$ be a block in $I$. Let $\Pi:=\{ \pi_j=1/2^{j-1}\}_{j=1,\ldots,Q}$ be a set of quantized probability thresholds and let
$$
\Upsilon := \left\lbrace \, \bp\!=\!(p_1, \ldots, p_N)  \, \mid \,   p_i \in \Pi  , \quad p_i \leqslant p_j \; \mathrm{if} \; i<j \right\rbrace
$$
be the family of non-decreasing $N$-tuples in $\Pi^N$, endowed with the  order
$\ba\geqslant \bb$ if and only if $a_i\geqslant b_i$ for all $i$.
The quantized  probability sequence associated with the event that random block $\bB$ resembles $B_{\bq}$ as closely as $B_{\bq'}$  does in the ith component  is defined by
\begin{equation}
(p^i_{\bq\,\bq'})_{i=1, \dots N} = \underset{t \in \Upsilon}{\inf} \lbrace t\; \mid\; t\geqslant  (\widehat{p^{i}}_{\bq\,\bq'})_{i=1,\dots N} \rbrace \,.
\end{equation}
\end{definition}
Notice that the infimum   $(p^1_{\bq \,\bq'}, \ldots, p^N_{\bq\,\bq'})
 $ is uniquely defined and belongs to $\Upsilon$. Put another way the quantized probability vector $(p^1_{\bq \,\bq'}, \ldots, p^N_{\bq\,\bq'})$ is the smallest upper bound of the resemblance probabilities $(\widehat{p^{1}}_{\bq\,\bq'},\ldots,\widehat{p^{N}}_{\bq\,\bq'})$ that can be found in $\Upsilon$.
Fig. \ref{non_decreasing_proba} illustrates the quantized probabilities in two cases.

\begin{figure}[h]
\begin{center}
\centering
\includegraphics[scale=0.33]{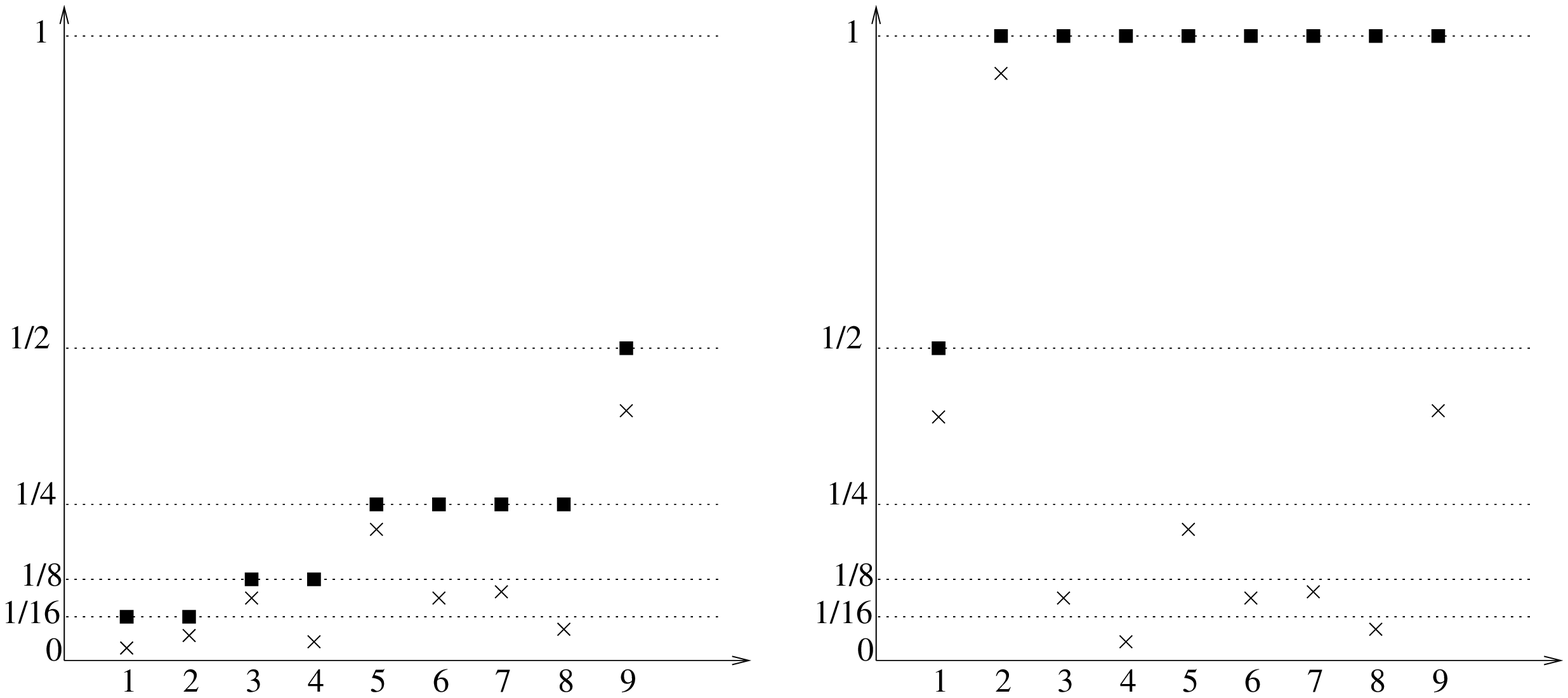}
\end{center}
\caption{Two examples of probabilities with $Q=5$ and $N=9$. The probability thresholds are in ordinate and the features in abscissa. The resemblance probabilities are represented with small crosses and quantized probabilities with small squares. The example on the left has a final probability of $\nicefrac{1}{(16^2 \cdot 8^2 \cdot 4^4 \cdot 2)}$. The right example has the same resemblance probabilities excepting for features $1$ and $2$, but the final probability is $\nicefrac{1}{2}$. Only the configuration on the left corresponds to a meaningful match.}
\label{non_decreasing_proba}
\end{figure}

\begin{proposition}[Quantized resemblance probability]\label{resemblanceprobability}
Let $B_{\bq} \in I$ and $B_{\bq'}  $ be two   blocks. Assume the principal components $i\in \{1, 2, \dots, s\}$ are reordered so that $|c_1(\bq)|\geqslant |c_2(\bq)|\geqslant \dots \geqslant |c_s(\bq)|$.  The  probability of the event {\rm
``the random block
 $\bB$ has its $N$ first components as similar to those of $B_{\bq}$ as to those of $B_{\bq'}$''} is
\begin{equation}
   Pr_{\bq\,\bq'} =   \prod_{i=1}^{N}p^{i}_{\bq\,\bq'}\;.
\end{equation}
\end{proposition}  This is a direct consequence of
  Def. \ref{defacontrariomodel}, the principal
components of $\bB$ being independent. The resemblance probability is the product of the marginal resemblance probabilities.
As classic in statistical decision, we could stop and use the above resemblance probability. But, despite having a low resemblance probability for each $Pr_{\bq\, \bq'}$, the large number of resemblance tests allows for a very large number of false matches. Our next goal therefore is to define a number of false alarms, and not a probability, as the right criterion. To this aim, we need to estimate the number of tests.

\subsection{Number of Tests}
The number of tests for comparing all the blocks of image $I$ with all the blocks in image $I'$ is the product of three factors. The first one is the image size
$\# I$. The second is the size of the search region denoted by $S'\subset I' $.  We mentioned before that the search is
done on the epipolar line. In practice, a segment of this line is enough. If $\bq=(q_1,q_2)$ is the point of reference it is enough to look for
$\bq'=(q'_1,q_2) \in I'$ such that $q'_1 \in [q_1-R,q_1+R]$ where $R$ is a fixed integer larger than the maximal possible disparity. The third and most important factor is the number of different non-decreasing probability distributions $FC_{N,Q}=\#\Upsilon$ that can be envisaged. Of course not all of these tests are performed, but only the one indicated by the observed block $B_{\bq'}$. Yet, the choice of  this unique test is steered by an {\it a posteriori} observation, while the calculation of the expectation of the number of false alarms (NFA) must be calculated {\it a priori}. Thus we must compute the NFA as though all comparisons for all quantized decreasing probabilities were effectuated.
 A test can never be defined {\it a posteriori},   it cannot be steered by the observation. Thus the number of tests is not the number of tests effectively  performed. There are $\#\Upsilon$  ways  each couple of blocks could {\it a priori} be  compared. In other terms  $\#\Upsilon$  different distances are {\it a priori} tested. Theorem 1 will ultimately justify the following definition.

\begin{definition}[Number of tests]\label{defnumberoftest} With the above notation we call the number of tests for matching two images $I$ and $I'$  the integer
$
    N_{test}= \#I \cdot \#S' \cdot \# \Upsilon \; = n \,(2R+1) \, FC_{N,Q}.
$
\end{definition}

\begin{lemma}
With the above notation,
\begin{equation}
FC_{N,Q}=\sum_{t=0}^{Q-1} (t+1)\cdot \binom{N+Q-t-3}{Q-t-1} \;,
\end{equation}
\end{lemma}
where
$$
FC_{N,Q} :=\# \{ \,f:[1,N]\rightarrow [1,Q] \, \mid\, f(x)\leqslant f(y),\, \forall x \leq y \, \} .
$$

In order to prove this result we write
\begin{align}
\overline{FC}_{N,Q} :=\# \{ \, & f:[1,N]\rightarrow [1,Q] \, \mid \,f(1)=1, \; f(N)=Q; \notag \\
 &  f(x)\leqslant f(y),\, \forall x \leqslant y \, \} \;. \notag
\end{align}

Since $\displaystyle{FC_{N,Q}=\sum_{t=0}^{Q-1} (t+1)\overline{FC}_{N,Q-t}}$ and\\ $\overline{FC}_{N,Q}=\binom{N+Q-3}{Q-1}$ the result follows.

We are now in a position to define a number of false alarms, which will control the overall number of false detections on the whole image.

\begin{definition}[Number of false alarms]\label{defNFA}
Let $B_{\bq} \in I$ and $B_{\bq'} \in I'$ be two observed  blocks. Assume the principal components $i\in \{1, 2, \dots, s\}$ are reordered so that $|c_1(\bq)|\geqslant |c_2(\bq)|\geqslant \dots \geqslant|c_s(\bq)|$.  We define the  Number of False Alarms associated with event {\rm
``the random block
 $\bB$ has its $N$ first components as similar to those of $B_{\bq}$ as those of $B_{\bq'}$ are''} by
$$
    NFA_{\bq,\bq'} = N_{test}  \cdot Pr_{\bq\,\bq'}= N_{test}\cdot \prod_{i=1}^{N}p_{\bq\,\bq'}^{i},
$$
where $N_{test}$ comes form Def. \ref{defnumberoftest} and $Pr_{\bq\,\bq'}$ is the probability that the random block $\bB$  have its first $N$ PCA components as similar to those of
$B_{\bq}$  as those of  $B_{\bq'}$ are (Prop. \ref{resemblanceprobability}).
\end{definition}

\begin{definition}[$\epsilon$-meaningful match]\label{def:meaningful_match}
A pair of pixels $\bq$ and $\bq'$ in a stereo pair $(I, I')$ is an $\epsilon$-meaningful match if
\begin{equation}
    NFA_{\bq\,\bq'} \leqslant  \epsilon \;.
\end{equation}
\end{definition}


\subsection{The Main Theorem}
 As it is computed above the NFA dimensionality is that of  a number (of false alarms) {\it per image}. An alternative would be to  measure the NFA as a number of false alarms per pixel, in which case the number of tests would not contain the cardinality of the image factor $\# I$.
  With the proposed NFA, it is up to the users to decide which number of false alarms per image they consider tolerable.
  The NFA of a match actually gives a security level: the smaller the NFA, the more meaningful the match intuitively is.  But Thm. \ref{Laseuleproposition} will give the real meaning of the  NFA. To state it, we will use a clever trick used by  Shannon in his information theory \cite{shannon2001mathematical}, page 22-23, namely to treat the probability  of an event as random variable and to play with its expectation. Here the NFA will become a random variable, replacing $B_{\bq'}$ with $\bB$ in its definition.

In the {\it a contrario} model, each comparison of $B_\bq$ with some $B_{\bq'}$ is interpreted as a comparison of $B_\bq$ to a trial of the random block model $\bB$. In total, $B_\bq$ is compared with $2R+1 other blocks$ for each $\bq\in I$. So, we are led to distinguish for each $\bq$  $(2R+1)$ trials which are as many i.i.d. random blocks $\bB^{\bq, j}$, $j\in\{1, 2, \dots 2R+1\}$, all with the same law as $\bB$. They model {\it a contrario} the $(2R+1)$ trials by which $B_\bq$ is matched to $(2R+1)$ blocks in $I'$. We are interested in the expectation of the number of such  trials being successful  (i.e. $\varepsilon$-meaningful), ``just by chance.''

Consider the  event $E_{\bq,j}$ that a random block $\bB^{\bq, j}$ in the {\it a contrario} model with reference image  $I'$ meaningfully matches $B_\bq$. If this happens, it is obviously a {\it false alarm}. We shall denote
by $\chi_{\bq,j}$ the random characteristic function associated with this event, with the convention that $\chi_{\bq, j}=1$ if $E_{\bq, j}$ is true, $\chi_{\bq, j}=0$ otherwise.
Similarly $NFA_{\bq,j}$ and $p^i_{\bq,j}$ are the NFA and quantized probabilities associated with the event $E_{\bq,j}$.

\begin{theorem}\label{Laseuleproposition}
Let  $\Gamma=\Sigma_{\bq\in I, j\in \{1, \dots, 2R+1\}} \chi_{\bq, j}$ be the random variable representing the number  of occurrences of an
$\epsilon$-meaningful match between a deterministic patch in the first image and a random patch in the second image.
Then the expectation of $\Gamma$ is less than or equal to $\epsilon$.
\end{theorem}

\begin{proof}

We have
\begin{equation*}
\chi_{\bq, j} = \left\{
\begin{array}{cc}
1, & \; \text{if } NFA_{\bq, j}\leqslant\epsilon;\\
0, & \; \text{if } NFA_{\bq, j} > \epsilon.
\end{array} \right.
\end{equation*}
Then, by the linearity of the expectation
$$ \mathbb{E} [\Gamma] = \sum_{\bq, j} \mathbb{E}[\chi_{\bq,j}] = \sum_{\bq, j} \Proba \left[ NFA_{\bq, j} \leqslant \epsilon  \right]. $$

The probability inside the above sum can be computed by  Definitions \ref{defNFA} and
\ref{defempiricalprobability}:
\begin{equation*}
 \mathbb{P} \left[ NFA_{\bq, j} \leqslant \epsilon \right]
 = \mathbb{P} \left[ \, \prod_{i}^{N} p^{i}_{\bq, j} \leqslant \frac{\epsilon}{N_{test}} \, \right]
\end{equation*} There are many probability $N$-tuples
 $p=(p^i_{\bq, j})_{i=1,\dots, N}$ permitting to obtain the inequality inside the above probability. Nevertheless, the probabilities having  been quantized, we can reduce it to a (non-disjoint) union of events, namely all  $p \in \Upsilon$ such that $ \prod_{i} p_i \leqslant \epsilon/N_{test}$.  By the Bonferroni correction the considered probability can be upper-bounded by the sum of their probabilities sum. In addition the intersection below
involves only independent events according to our background model. Thus
\begin{align*}
\mathbb{P} \left[  \prod_{i}^{N} p^{i}_{\bq, j} \leqslant \frac{\epsilon}{N_{test}}\right]
& = \mathbb{P} \left[ \bigcup_{\substack{ p \in \Upsilon\\ \prod_{i} p_i \leqslant \epsilon/N_{test} }}  \bigcap_{i} \big(p^{i}_{\bq, j}\leqslant  p_i \big) \right]\\
&\leqslant \sum_{\substack{p \in \Upsilon \\ \prod_{i} p_i \leqslant \epsilon/N_{test} }} \prod_{i}    p_i    \\
&  \leqslant \; \dfrac{\epsilon}{\#I \, \#S'},\\
\end{align*}
where we have also used  $N_{tests} =
\#I \, \#S' \, \#\Upsilon$.
So we have shown that
$$
\mathbb{E}[\Gamma] = \sum_{\bq, j} \mathbb{E} \left[ \chi_{\bq, j} \right] \leqslant \sum_{\bq,  j}
\dfrac{\epsilon}{\#I \, \#S'} = \epsilon.
$$

\end{proof}

The $\epsilon$ parameter is the only legitimate parameter of the method, the other ones namely the block size $\sqrt s$, the number of principal components $N$ and the number of quantized probability thresholds $Q$ can be fixed once and for all for a given SNR (Signal to Noise Ratio). All experiments are made with a common SNR, but a lower SNR would allow smaller blocks and consequently a different set of parameters.  The question of how many false alarms should be acceptable in a stereo pair depends on the size of the images. In all experiments with moderate size images, of the order of $10^6$ pixels,  the decision was to fix $\varepsilon=1$. Thanks to Theorem \ref{Laseuleproposition} this means that it is expected to find one false alarm in average for images with $10^6$ pixels. Then, fixing $\varepsilon$ makes the method into a parameterless method for all moderately sized images.

\section{The Self-Similarity Threshold}
\label{sec:autosimilarity-threshold}

Urban environments contain many periodic local structures (for example the windows on a fa\c cade). Since, in general, the number of repetitions is insignificant with respect to the number of blocks that have been used to estimate the empirical {\it a contrario} probability distributions, the
{\it a contrario} model does not learn this repetition, and can be fooled by such repetitions, thus signaling a
significant match for each repetition of the same structure. Of course, one of those significant matches is the correct
one, but chances are that the correct one is not also the most significant. In such a situation two choices are left:
\emph{(i)} try to match the whole set of self-similar blocks of $I$ as a single multi-block (typically, global
methods such as graph-cuts do that implicitly); or \emph{(ii)} remove any (probably wrong) response in the case where
the stroboscopic effect is detected. The first alternative would lead to errors anyway, if the similar blocks do not have
the same height, or if some of them are out of field in one of the images. Fortunately, stereo pair block-matching yields a straightforward  adaptive threshold. A distance function $d$ between blocks being defined, let $\bq$ and $\bq'$ be points in the reference and secondary images respectively that are candidates to match with each other. The match of $\bq$ and $\bq'$ will be accepted if the following self-similarity (SS) condition is satisfied:
\begin{equation}\label{test_ss}
 d(B_{\bq},B_{\bq'}) < min\{ d(B_{\bq},B_{\br}) | \; r \in I \cap S(\bq)\}
\end{equation}
where $S(\bq)=[q_1-R\,,\,q_1+R] \, \backslash \,\{q_1,\,q_1+1,\,q_1-1\}$ and $R$ is the search range.
 As noted earlier, the search for correspondences can be restricted to the epipolar line. This is why the automatic threshold is restricted to $S(\bq)$. The distance used in the self-similarity threshold is the sum of squared differences (SSD) of all the pixels in the block and the block size is the same than the block size use for ACBM.

Computing the similarity of matches in one of the images is not a new idea in stereovision. In \cite{Manduchi99} the authors define the \textit{distinctiveness} of an image point $\bq$ as the perceptual distance to the most similar point other than itself in the search window. In particular, they study the case of the auto-SSD function (Sum of Squared Differences computed in the same image). The flatness of the function contains the expected match accuracy and the height of the smallest minimum of the auto-SSD function beside the one in the origin gives the risk of mismatch. They are able to match  ambiguous points correctly by matching intrinsic curves \cite{Tomasi98}.
However, the proposed algorithm only accepts  matches when their quality is above a certain threshold. The obtained disparity maps are rather sparse and the accepted matches are completely concentrated on the edges of the image.
According to \cite{Sara02}, the ambiguous correspondences should be rejected. In this work a new \textit{stability property} is defined. This property is one condition a set of matches must satisfy to be considered unambiguous at a given confidence level. The stability constraint and the tuning of two parameters permits to take care of flat or periodic autocorrelation functions.  The comparison of this last algorithm with our results will be done in section \ref{sec:experimental_results}.

\subsection{\textit{A Contrario} vs Self-Similarity}

Is the self-similarity  (SS) threshold really necessary? One may wonder whether the \textit{a contrario} decision rule to accept or reject correspondences between patches would be sufficient by itself. Conversely, is the self-similarity threshold  enough to reject false matches in a correlation algorithm? This section addresses both questions and  analyzes some simple examples  enlightening the  necessity and complementarity of both tests. For each example we are going to compare the result of the {\it a contrario} test and the result of a classic correlation algorithm combined with the self-similarity threshold alone.

First consider two independent Gaussian noise images (Fig. \ref{fig:noise}). It is obvious that we would like to reject any possible match between these two images. As expected, (this is a sanity check!) the \textit{a contrario} test rejects all the possible patch matches. On the other hand, the correlation algorithm combined with the self-similarity is not sufficient:  many false matches are accepted.

\begin{figure}[h]
\begin{center}
\begin{minipage}{.3\columnwidth}
\centering
  \includegraphics[height=2.5cm]{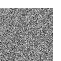}\\
	(a)\\
\end{minipage}
\hfil
\begin{minipage}{.3\columnwidth}
\centering
   \includegraphics[height=2.5cm]{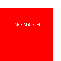}\\
	(b)\\
\end{minipage}
\hfil
\begin{minipage}{.3\columnwidth}
\centering
  \includegraphics[height=2.5cm]{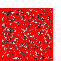}\\
	(c)\\
\end{minipage}
\end{center}
\caption{(a) Reference noise image. (b) No match at all has been accepted by the \textit{a contrario}  test! (c) Many false correspondences have been accepted by the self-similarity threshold.}
\label{fig:noise}
\end{figure}

The second comparative test is about occlusions. If a point of the scene can be observed in only one of the images of the stereo pair, then an estimation of its disparity is simply impossible. The best decision is to reject its matches. A good example to illustrate the performance of both rejection tests ACBM and SS is the map image (Middlebury stereovision database, Fig. \ref{map_ac_as}) which has a large baseline and therefore an important number of occluded pixels. ACBM gives again the best result (see Table \ref{table_map_ac_as}).  The table indicates that the self-similarity test only removes a few additional points. Yet, even if the proportion of eliminated points is tiny, such mismatches can be very annoying and the gain is not negligible at all.

\begin{figure}[h]
\begin{center}
\begin{minipage}{.45\columnwidth}
\centering
  \includegraphics[height=3cm]{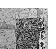}\\
	(a)\\
\includegraphics[height=3cm]{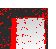}\\
	(c)\\
\end{minipage}
\hspace{.2cm}
\begin{minipage}{.45\columnwidth}
\centering
  \includegraphics[height=3cm]{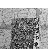}\\
	(b)\\
\includegraphics[height=3cm]{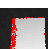}\\
	(d)\\
\end{minipage}

\end{center}
\caption{(a) Reference image (b) Secondary image. The rectangular object occludes part of the background (c) The \textit{a contrario} test does not accept any match for pixels in the occluded areas.  (d) With the self-similarity threshold the disparity map is denser, but wrong disparities remain in the occluded region.}
\label{map_ac_as}
\end{figure}

\begin{table}

 \begin{center}
  \begin{tabular}{c|c|c}
   & Bad matches & Total matches\\
\hline
SS &  3.35\%     &  85.86\%   \\
ACBM &  0.37\%     &  64.85\%    \\
ACBM+SS &0.36\%    &  64.87\%    \\
\end{tabular}

 \end{center}
\caption{Quantitative comparison of several algorithms on Middlebury's Map image: the block-matching algorithm with the self-similarity threshold (SS), the \textit{a contrario} algorithm (ACBM) and the algorithm  combining both (ACBM+SS). The percentage of matches for each algorithm is computed in the whole image and among these the number of wrong matches is also given. A match is considered wrong if its disparity difference with the ground truth disparity is larger than one pixel.}
\label{table_map_ac_as}
\end{table}

 The {\it a contrario} methodology cannot detect the ambiguity inherent in periodic patterns. Indeed, periodicity certainly does not occur ``just by chance.'' The match between a window and another identical window on a building fa\c cade is obviously non casual and is therefore legally accepted by an {\it a contrario} model.   In this situation, the self-similarity test is necessary. A synthetic case has been considered in Fig. \ref{brodatz_ratlla}, where the accepted correspondences are completely wrong in the \textit{a contrario} test for the repeated lines. On the contrary, the self-similarity threshold is able to reject matches in this region of the image.

\begin{figure}
\begin{center}
\begin{minipage}{.31\columnwidth}
\centering
  \includegraphics[height=2.4cm,width=2.4cm]{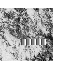}\\
	(a)\\
\end{minipage}
\begin{minipage}{.31\columnwidth}
\centering
  \includegraphics[height=2.5cm,width=2.5cm]{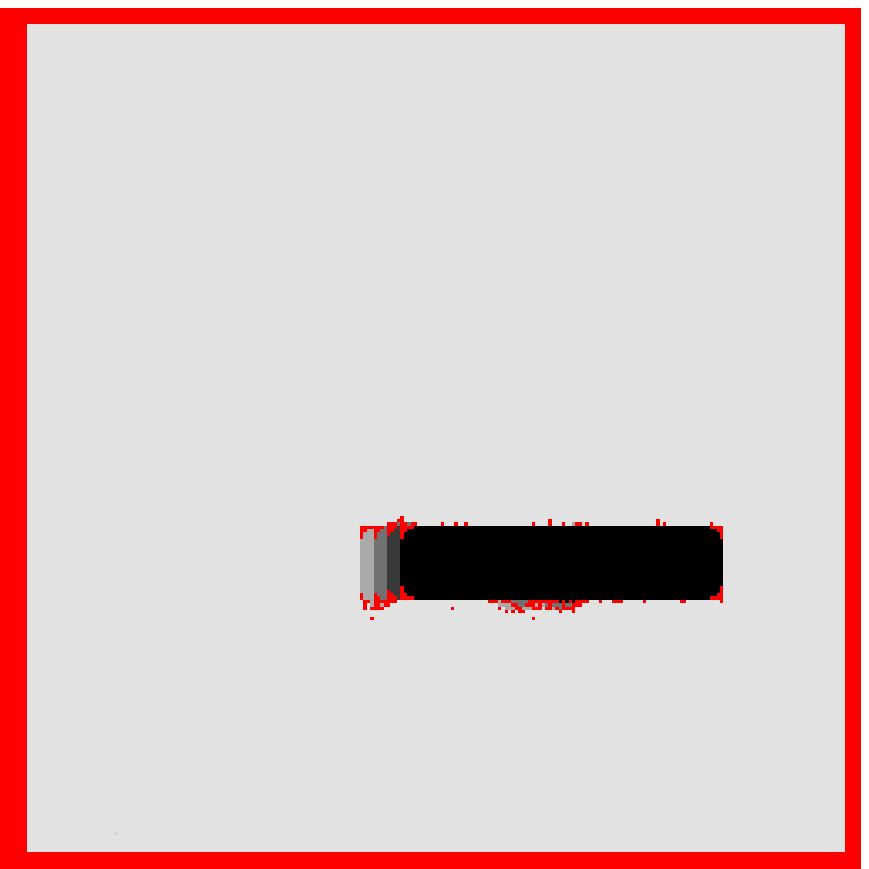}\\
	(b)\\
\end{minipage}
\begin{minipage}{.31\columnwidth}
\centering
   \includegraphics[height=2.5cm,width=2.5cm]{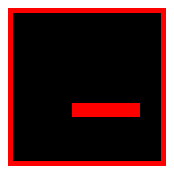}\\
	(c)\\
\end{minipage}

\end{center}
\caption{(a) Reference image with a texture and a stripes periodic motif. The secondary image is a  2 pixels translation of the reference image. The obtained disparity map should be a constant image with value 2. (b) The \textit{a contrario} test gives the right disparity 2 everywhere, except in the stripes region. (c) The repeated stripes are locally similar, so the self-similarity threshold rejects all the patches in this region. }
\label{brodatz_ratlla}
\end{figure}

In short, ACBM and SS are both necessary and complementary.   SS only removes a tiny additional number of errors, but even a few outliers can be very annoying in stereo. From now on, a possible match $(\bq, \bq')$ will therefore be accepted only if it is a meaningful match (ACBM test in Def.~\ref{def:meaningful_match}) and satisfies the SS condition given by (\ref{test_ss}).

\section{Comparative Results}

\label{sec:experimental_results}
 The algorithm parameters are identical for all experiments throughout this paper. The comparison window size is   $9\times9$, the number of considered principal components is $9$, the number of quantum probabilities  is $5$.  The previous section showed how the proposed method (ACBM + SS) deals with noise, occlusions and repeated structures. The detection method is also adapted to quasi-simultaneous stereo from aerial or satellite images, where moving  objects (cars, pedestrians) are a serious disturbance. Essentially, this is the same problem as the occlusion problem, but the occlusion is caused by camera motion in presence of a depth difference instead of object motion.
Figure~\ref{marseille} shows  a stereo pair of images of the city of Marseille (France). In both cases, several cars have changed position between the two images. They are duly detected. The shadow regions, which contain more noise than signal, have also been rejected. We have also compared our results with the Kolmogorov's graph cut implementation \cite{Kolmogorov05} which rejects {\it a posteriori} incoherent matches and are labeled as occlusions. In these examples, graph cuts is able to reject some mismatches due to the moving objects in the scene but a lot of conspicuous errors remain in the final disparity map. Likewise, OpenCV's stereo matching algorithm \cite{opencv} fails completely on this kind of pairs, even though it obtains correct results in more simple examples like the one in figure~\ref{map_ac_as}.
\\


\begin{figure}
\centering
 \begin{tabular}{ccc}
 \vcell{reference image} &
 \includegraphics[width=3cm]{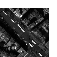} & \includegraphics[width=3cm]{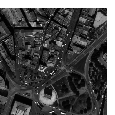} \\
 \vcell{secondary image} &
 \includegraphics[width=3cm]{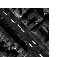} & \includegraphics[width=3cm]{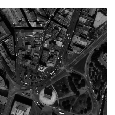}\\
 \vcell{ACBM+SS} &
 \includegraphics[width=3cm]{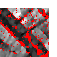} & \includegraphics[width=3cm]{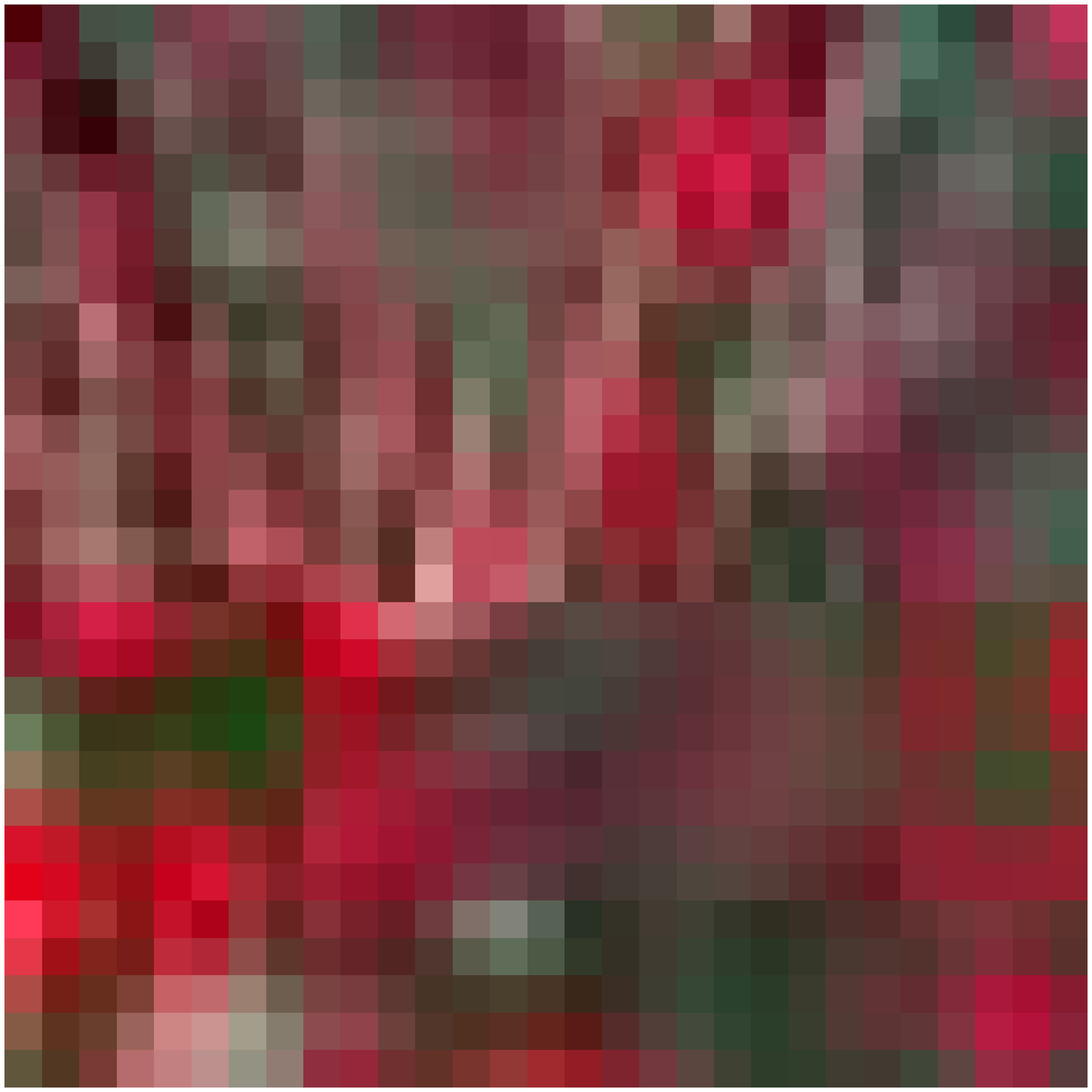}\\
 \vcell{graph-cuts} &
 \includegraphics[width=3cm]{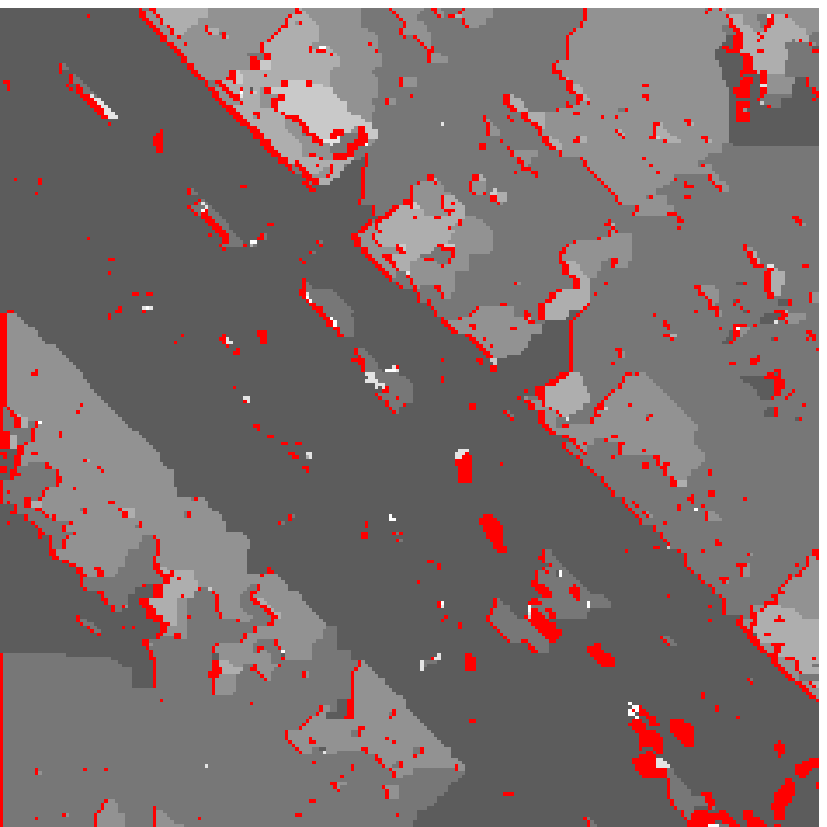} & \includegraphics[width=3cm]{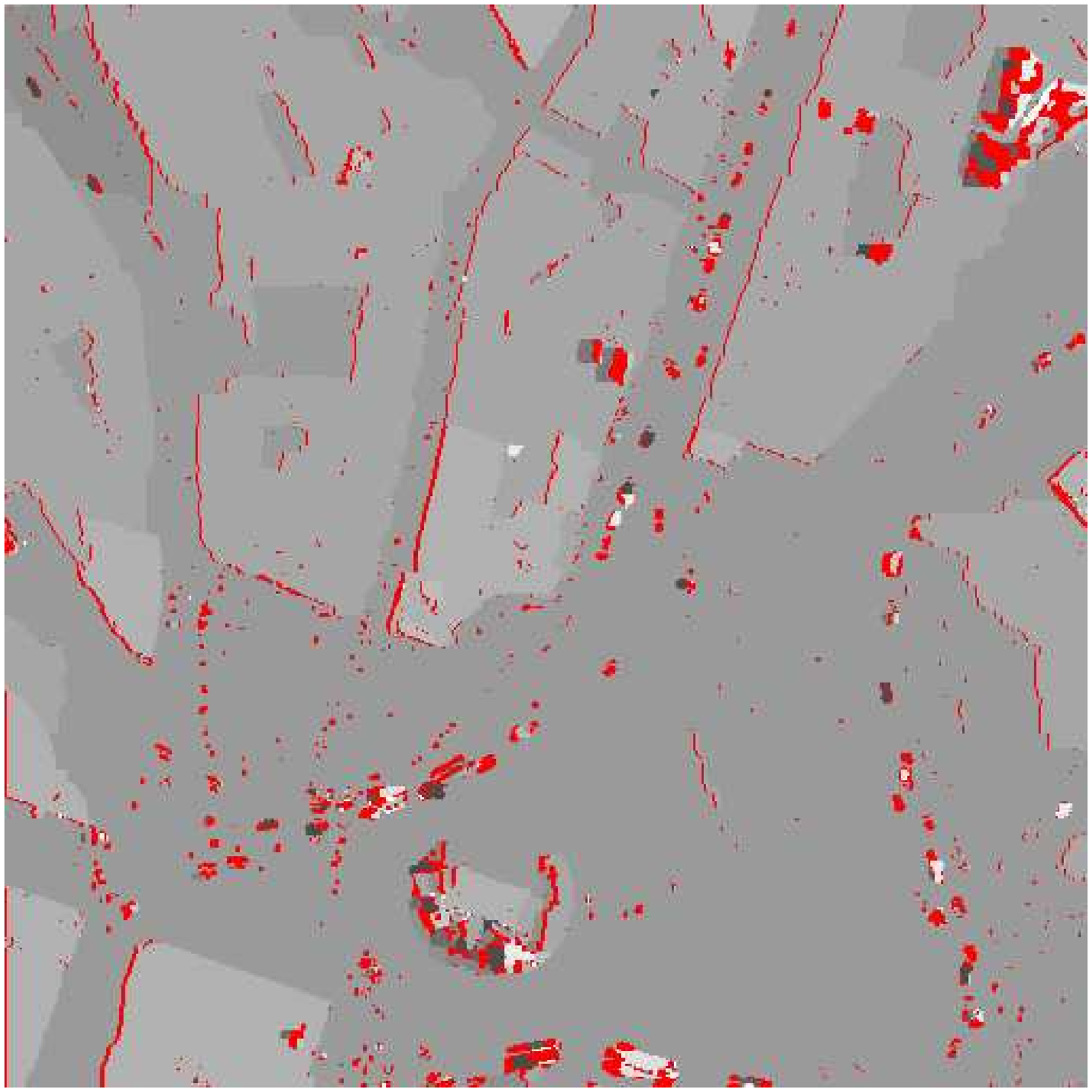}\\
 \vcell{OpenCV} &
 \includegraphics[width=3cm]{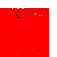} & \includegraphics[width=3cm]{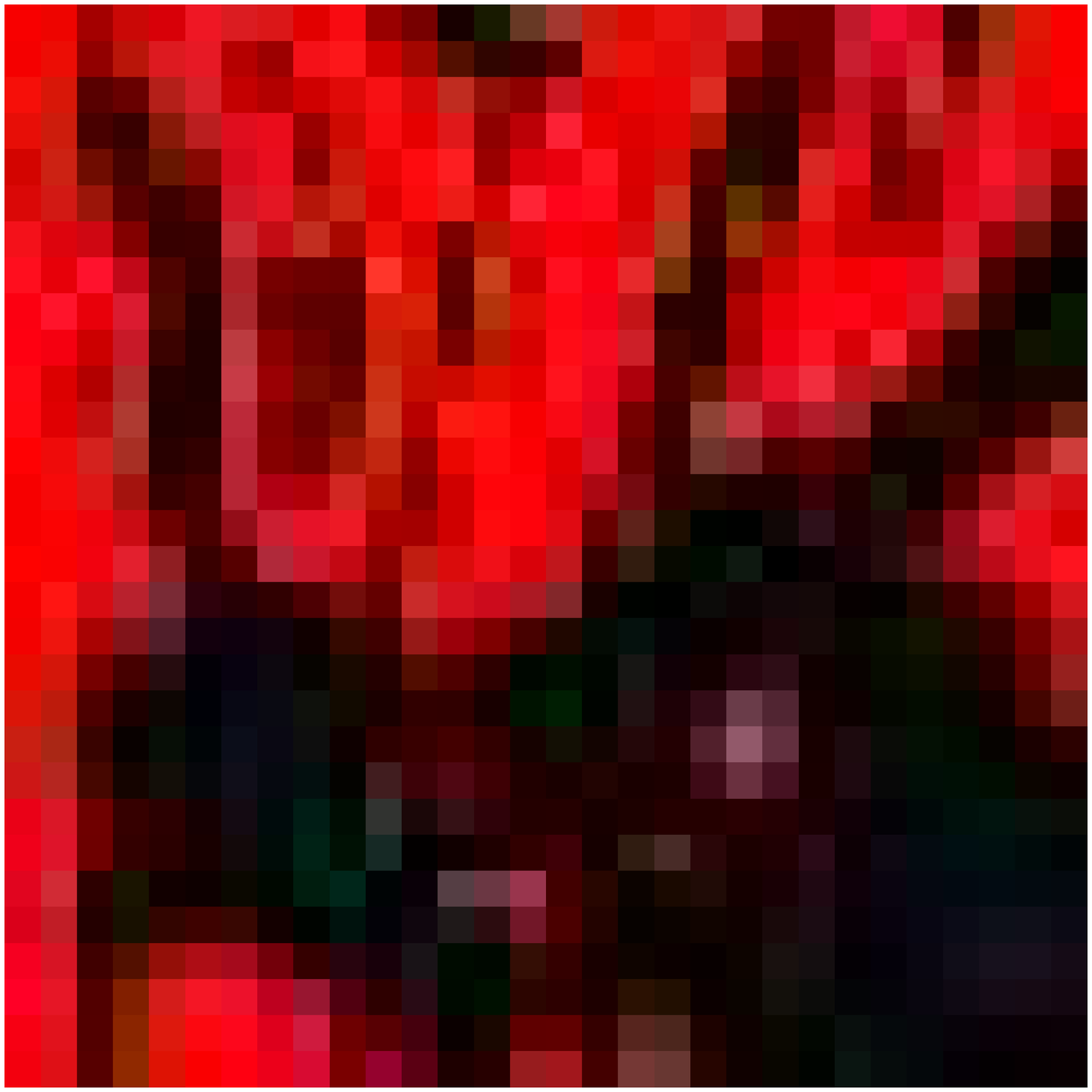}\\
 \end{tabular}
\caption{ From top to bottom:  reference image, secondary image, ACBM+SS disparity map, graph cuts disparity map, and OpenCV disparity map. In our disparity map, red points are points which haven't been matched. Notice that patches containing a moved car or bus haven't been matched. Poorly textured regions (shadows) where noise dominates have  also been rejected. Red points in the graph cuts disparity map are rejected {\it a posteriori} and considered as occlusions. The graph cuts disparity map is denser and smoother but several mismatches appear in the low textured areas and regions with moving objects.}
\label{marseille}

\end{figure}

The proposed algorithm will now be compared with the non-dense algorithms of \cite{Sara02}, \cite{Veksler02}, \cite{Veksler03} and \cite{Mordohai06}, whose aims are comparable. All of these papers have published experimental results on the first Middlebury dataset \cite{Scharstein02} (Tsukuba, Sawtooth, Venus and Map pair of images), on the non-occluded mask. These four algorithms  compute sparse disparity maps and propose techniques rejecting unreliable pixels.
We also show some additional comparison with the block matching method implemented in the OpenCV library version 2.2.0 \cite{opencv}, because it is possibly the most widely used one since it comes close to real-time performance.

The authors of \cite{Mordohai06} compute an initial classic correlation disparity map and select correct matches based on the support these pixels receive from their neighboring candidate matches in 3D after tensor voting. 3D points are grouped into smooth surfaces using color and geometric information and  the points which are inconsistent with the surface color distribution are removed. The rejection of wrong pixels is not complete, because the algorithm fails when some objects appear only in one image, or when occluded surfaces change orientation. A variation of the critical rejection parameters can lead to quite different results.

\cite{Veksler02} detects and matches so called ``dense features'' which consist of a connected set of pixels in the left image and a corresponding set of pixels in the right image such that the intensity edges on the boundary of these sets are stronger than their matching error on  the boundary (which is the absolute intensity difference between corresponding boundary pixels). They call this the ``boundary condition.'' The idea is that even the boundary of a non textured region can give a correspondence. Then, each dense feature is associated with a disparity. The main limitation is the way dense features are extracted. They are extracted using a local algorithm which processes each scan line independently from the others. As a result, top and bottom boundaries are lost. On the contrary, \cite{Veksler03} uses graph cuts to extract ``dense features'' (which of course does not necessarily imply a dense disparity map) thus enforcing the boundary conditions. The results in \cite{Veksler02} are rather dense and the error rate is one of the most competitive ones. Yet these good results are also due to the particularly well adapted structure of the benchmark.  Indeed, the Sawtooth, Venus and Map scenes consist of piecewise planar surfaces, with almost fronto-parallel surface patches. The ground truth of Tsukuba is a piecewise constant disparity map with six different disparities.

Table \ref{table_compar_sparse} summarizes the percentage of matched pixels (density) and the percentage of mismatches (where the estimated disparity differs by more than one pixel from the ground truth). This table  reports first the result of  ACBM+SS, whose error rate is very small and yields larger match densities than Sara's results \cite{Sara02}.
To compare with other algorithms yielding denser disparity maps, the results of ACBM+SS have been densified  by the most straightforward proximal interpolation (a 3$\times$3 spatial median filter). Doing this, the match density rises significantly while keeping small error rates. Still,  large regions containing poor textures, typically shadows in aerial imaging, are impossible to fill in because they contain no information at all.  Besides the compared algorithms in Table \ref{table_compar_sparse} \cite{Szeliski01} also published non-dense results for the Tsukuba image (error rate of $2.1\%$ with a density of $45\%$) but since non-dense results on other images are not published it does not appear in our table.

Fig. \ref{fig:shrub} compares the ACBM+SS results with opencv, graph cuts and the Sara published results on the classic CMU Shrub pair\footnote{http://vasc.ri.cmu.edu/idb/html/jisct}. Sara's disparity map has several mismatches and the ACBM+SS results are obviously denser. On the other hand, Kolmogorov's graph cut implementation is denser but the mismatches have risen considerably. OpenCVs disparity map is more dense than Kolmogorov's, and less dense than Sara's, but it has also the highest number of wrong matches. So, the proposed algorithm ACBM+SS has a better trade-off between density and mismatches.  In the Kolmogorov graph cuts implementation the occlusions are detected, providing a non-dense disparity map. It is clear that detecting occlusions in real images is not enough to avoid mismatches. Another example is shown in Fig. \ref{fig:flower_garden}, where the almost dense disparity map obtained with graph cuts is compared with the ACBM+SS disparity map. The top left of the image gets by Graph Cuts a completely wrong disparity: the sky and the tree branches are clearly not at the same depth in the scene. This type of error is unavoidable with global methods. The depth of the smooth sky is inherently ambiguous. By the minimization process it inherits the depth of the twigs through which it is seen.

An interesting question arises out of the comparative results about the duality error/density. We have seen that our algorithm gives very low error percentages with densities between 40\% and 90\%. The parameter $\varepsilon$ can be increased but then the error rate will rise. Our goal is to match with high reliably the points between two images and reject any possible false match. So the choice of one expected false alarm ($\varepsilon=1$) is a conservative choice but ensures a very small error percentage.

{\it Discussion on the other parameters: }
We have mentioned that the number of considered principal components $N$ and the number of quantum probabilities $Q$ can be increased without noticeable alteration of the results. In practice, the two values are chosen (for computational reasons) to the minimal values not affecting the quality of the result.  They are fixed once and for all to $N=9$ and $Q=5$ respectively. Another parameter is the search region size ($2R+1$) but it is easy to find since we only need $R$ to be larger than the largest disparity in the image, which is a classic assumption in stereovision algorithms (in practice $R$ can be estimated from the sparse matching of interest points that was previously obtained for the epipolar rectification step). Finally, the last parameter is the size of the block. We know that very small blocks are affected by image noise but at the same time, the bigger the block, the bigger the fattening error (also named adhesion error). This error becomes apparent at the object borders of the scene causing a dilation of their real size, which is proportional to the block-size. The fattening phenomenon is not the object of the this paper but different solutions have already been suggested to avoid it \cite{Delon07}. Fixing the size of the block to $9 \times 9$ seems to be a good compromise between noise and fattening for a realistic SNR conditions, ranging from 200 to 20 (the SNR is measured as the ratio between the average grey level and the noise standard deviation.)

{\it Computational time: }
For the sake of computational speed, the PCA basis is previously learnt on a set of representative images and stored once and for all.\footnote{
In our experience the (computationally intensive) choice of this basis does not significantly affect the results, but the (computationally fast) learning of marginal distributions for a particular image on this basis  does.}
Then, this basis is used to compute all image coefficients. Notice that only the image coefficients of the second image need to be sorted in order to compute the resemblance probability between all possible matches. With our implementation, which is still not highly optimized for speed, an experiment with a pair of images of size $512 \times 512$ with disparity rang = [-5,5], takes 4.5 seconds running on a 2.4 GHz Intel Core 2 Duo processor. \\
A similar experiment with the OpenCV stereo algorithm takes between 5 and 500 miliseconds. This is much closer to real-time requirements, but results are also much more data-dependent, producing good results in easy examples like the Middlebury pair, but much less dense and less reliable results than our method in more difficult scenes like shrub, marseille or even the stereo pairs provided with OpenCV.


\begin{table*}
\begin{center}
\begin{tabular}{|c||c|c|c|c|c|c|c|c|}

\hline
        & \multicolumn{2}{c|}{Tsukuba} & \multicolumn{2}{c|}{Sawtooth} & \multicolumn{2}{c|}{Venus} &  \multicolumn{2}{c|}{Map} \\
\cline{2-9}
	& Error(\%) & Density(\%)  &  Error(\%) & Density(\%) &  Error(\%) & Density(\%) &  Error(\%) & Density(\%)\\
\hline
ACBM + SS   & \textbf{0.31}   &  45.6  & \textbf{0.09}  & 65.7 & 0.02 &  54.1   &  \textbf{0.0} &  84.8   \\
\hline
ACBM + SS + Median filter   & 0.33   &  54.3  & 0.14  & 77.9 &  \textbf{0.0} &  66.6   & \textbf{0.0}  &  93.0   \\
\hline
Sara   \cite{Sara02}       & 1.4   &  45  & 1.6 & 52  &  0.8 & 40   &  0.3 &  74   \\
\hline
Veksler 02  \cite{Veksler02}    & 0.38   &  66  & 1.62  & 76 & 1.83  &  68  & 0.22  &  87   \\
\hline
Veksler 03   \cite{Veksler02}   & 0.36   &  75  & 0.54  & 87 &  0.16 &  73  & 0.01  &  87   \\
\hline
Mordohai and Medioni \cite{Mordohai06} & 1.18   &  74.5  & 0.27  & 78.4 &  0.20 & 74.1   &  0.08 & 94.2    \\
\hline

\end{tabular}
\end{center}
\caption{Quantitative results on the first Middlebury benchmark data set. The error statistics are computed on the mask of non occluded pixels. Any error larger than $1$ pixel is considered a mismatch. ACBM+SS obtains less mismatches in all four images.}
\label{table_compar_sparse}
\end{table*}

\begin{figure}[h]
\begin{center}
\begin{minipage}{.45\columnwidth}
\centering
 \includegraphics[height=3cm]{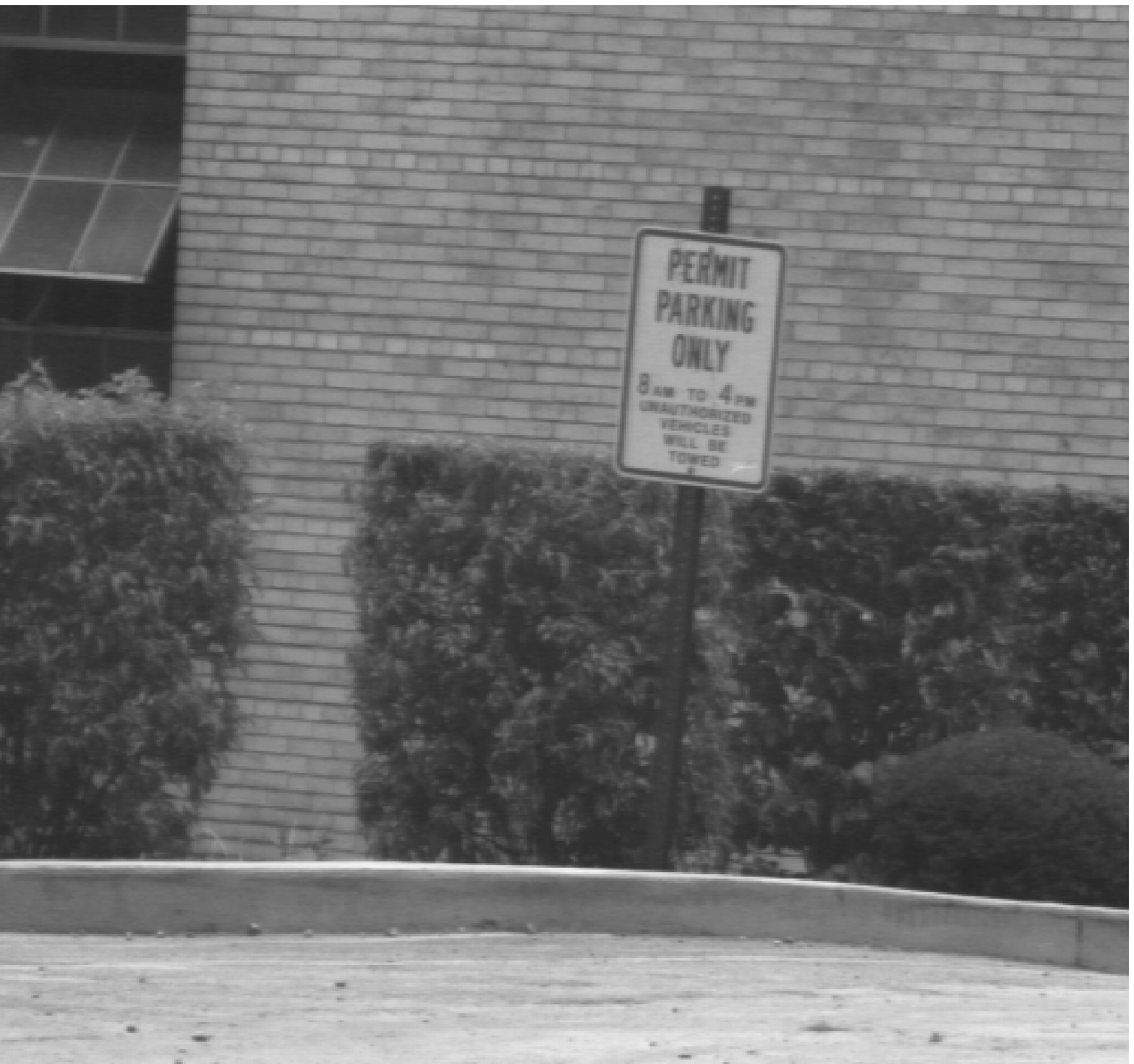}\\
	(a) left image\\ \vspace{0.1cm}
\includegraphics[height=3cm]{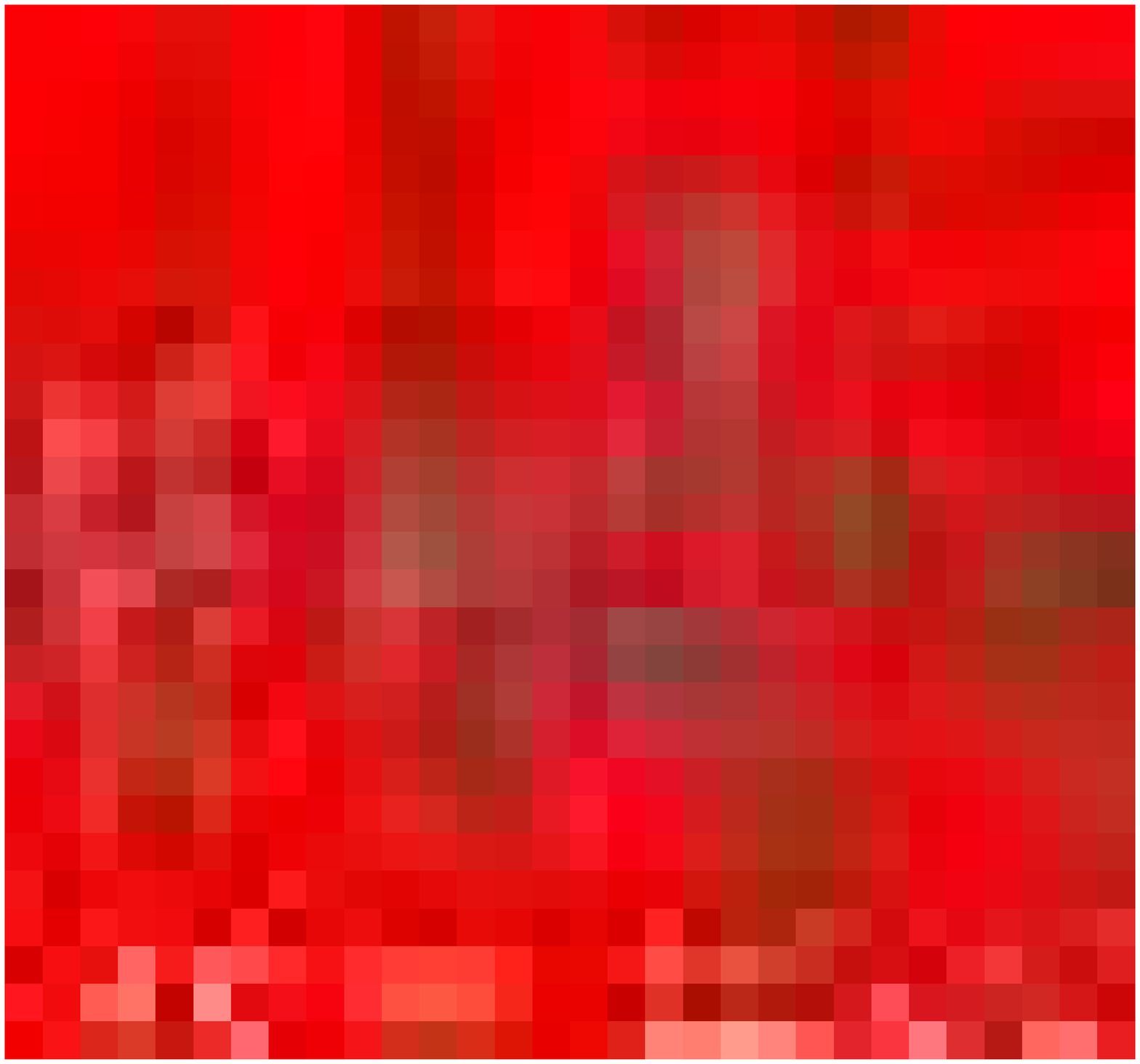}\\
	(c) Sara \cite{Sara02}\\ \vspace{0.1cm}
\includegraphics[height=3cm]{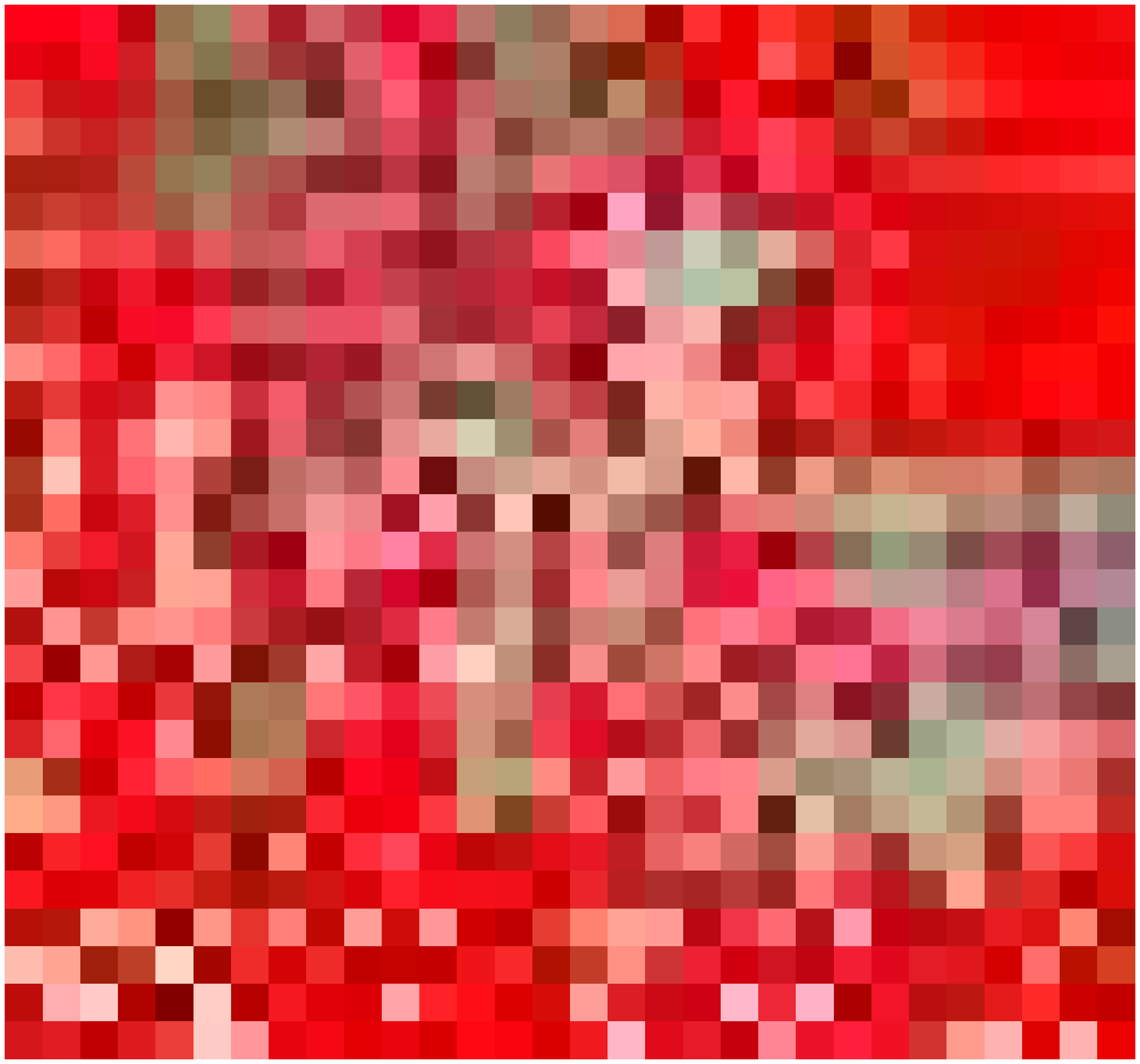}\\
(e) Proposed algorithm\\
\end{minipage}
\hfil
\begin{minipage}{.45\columnwidth}
\centering
\includegraphics[height=3cm]{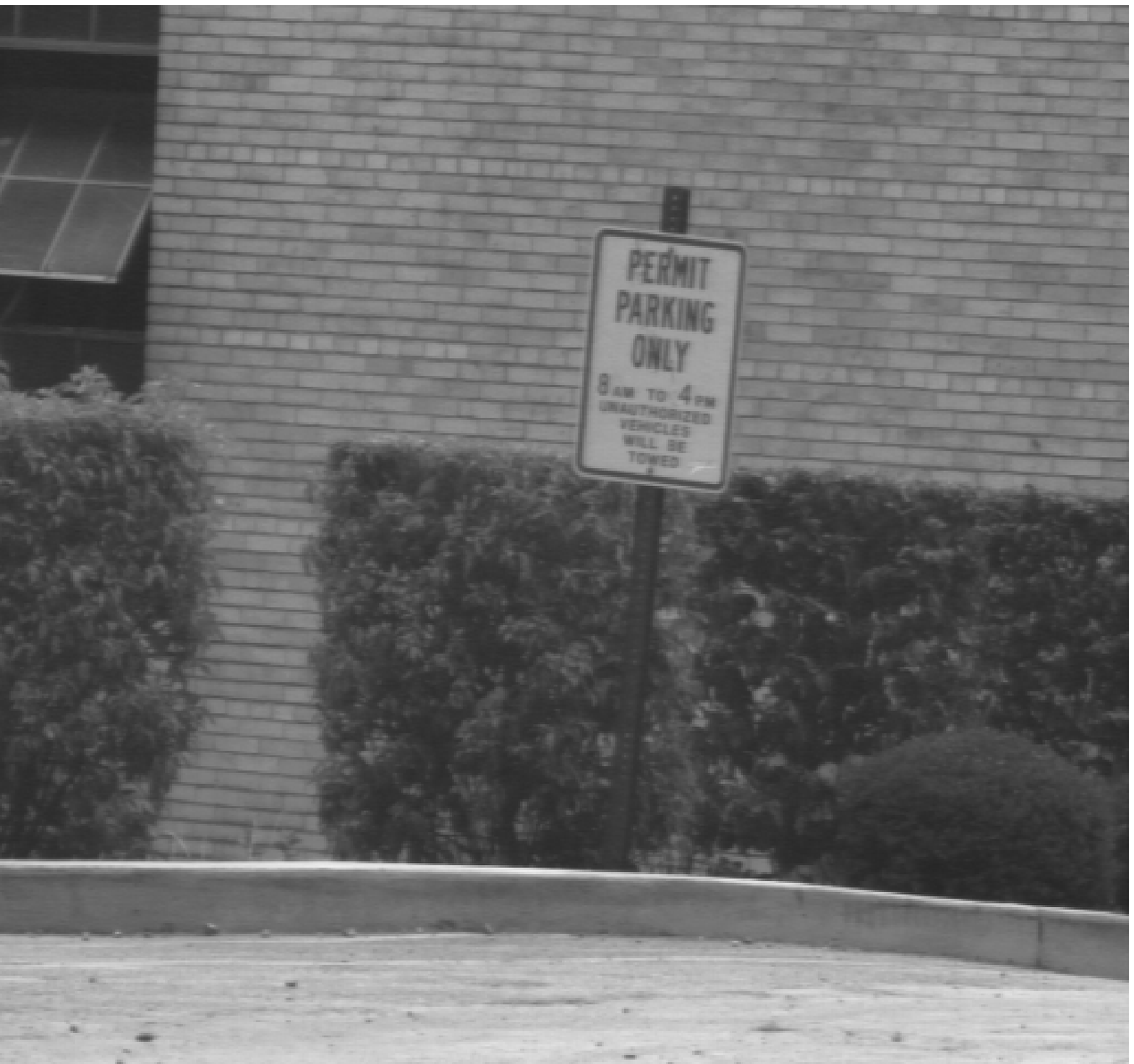}\\
	(b) right image\\ \vspace{0.1cm}
\includegraphics[height=3cm]{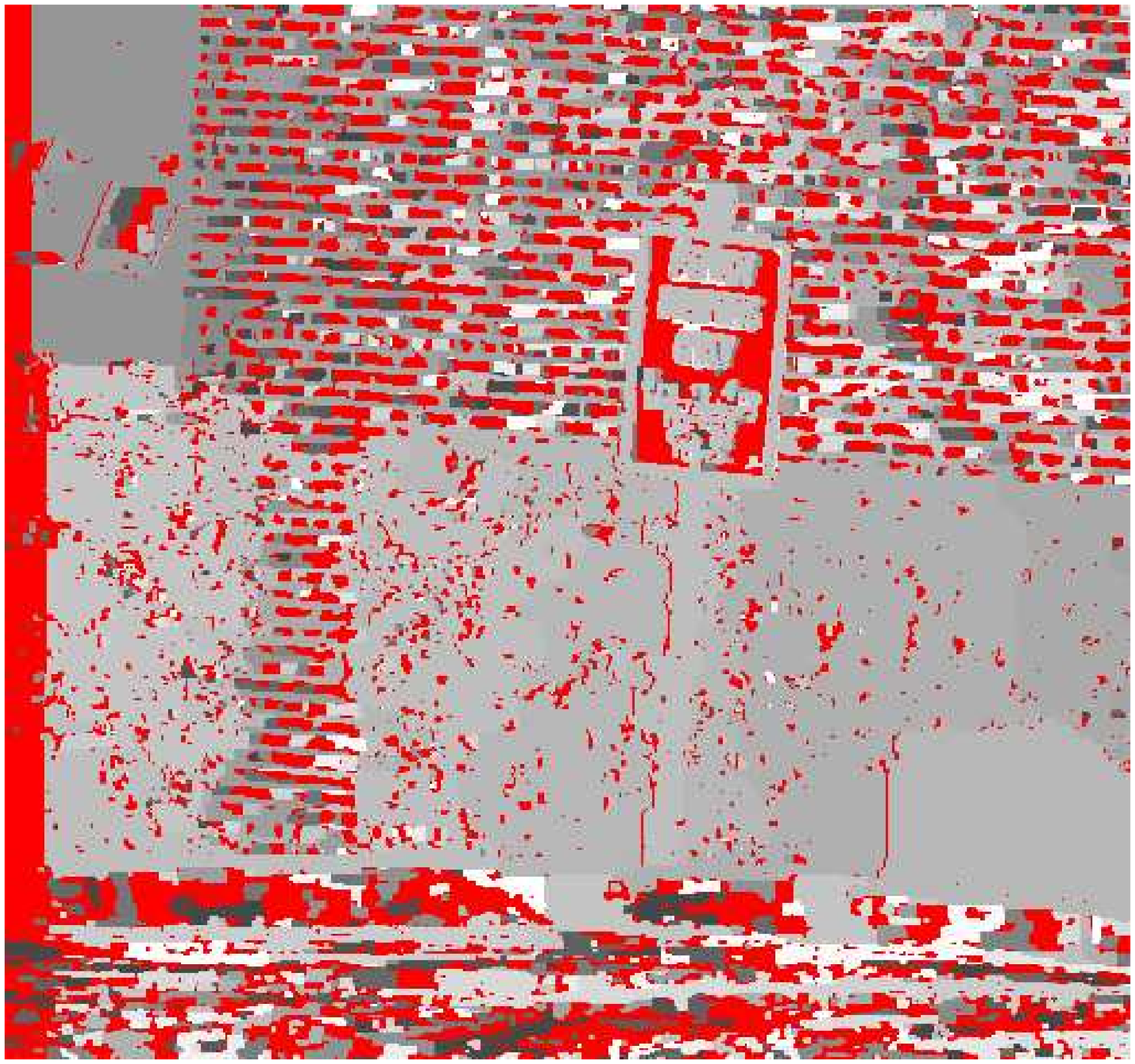}\\ 
	(d) Kolmogorov \cite{Kolmogorov05} \\ \vspace{0.1cm}
\includegraphics[height=3cm]{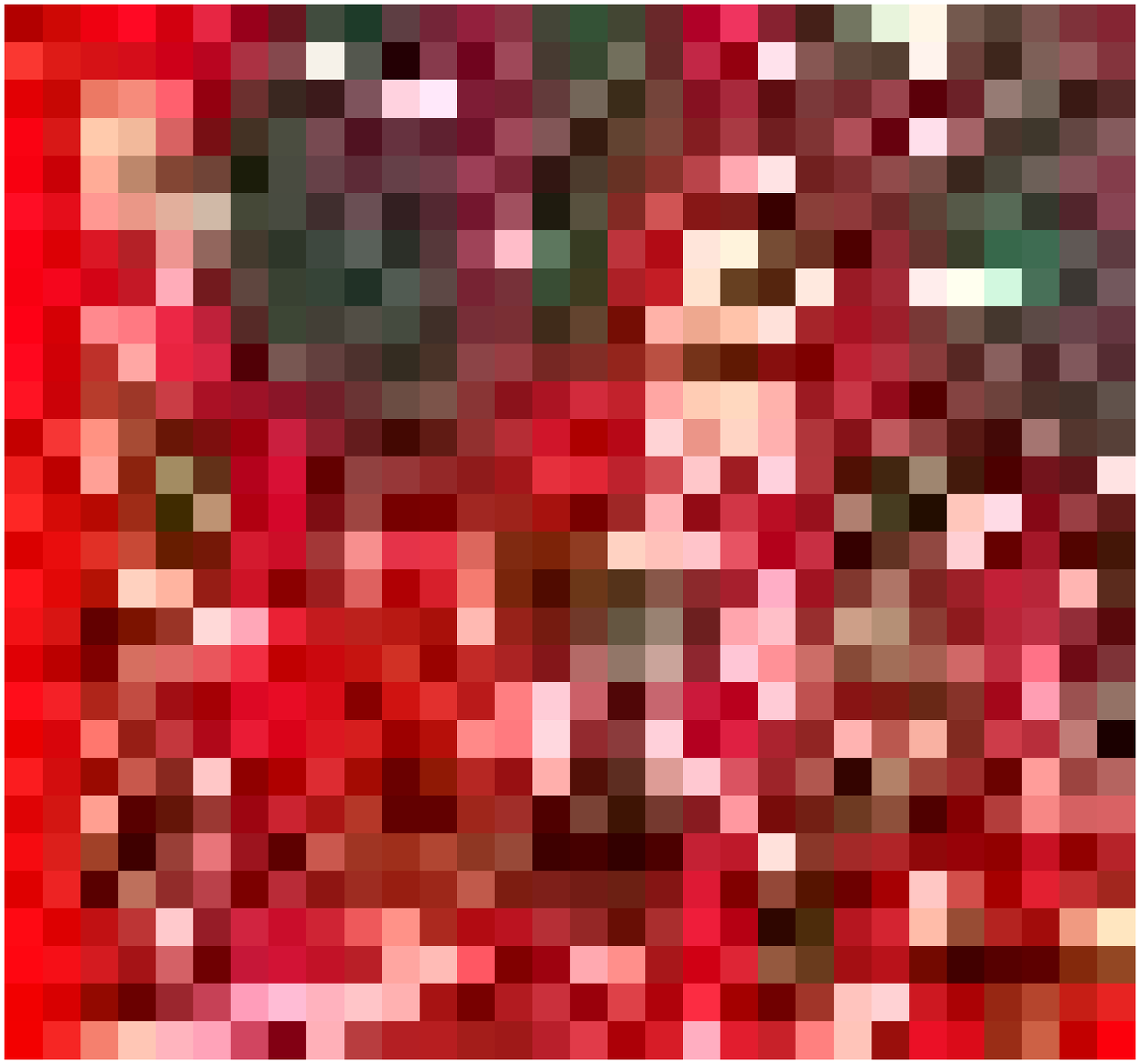}\\
   (f) OpenCV SGBM
\end{minipage}

\end{center}
\caption{CMU Shrub scene. (a) and (b) Reference and secondary images. (c) Method of Sara \cite{Sara02}. Red points are rejected. Density: 24\% (d) Kolmogorov's Graph-Cuts \cite{Kolmogorov05}. Red points are points detected as occlusions. Density: 77\%  (e) ACBM+SS. Red points are rejected points. Density: 42\%. Sara's disparity map has a lower density and has several evident mismatches. Kolmogorov's disparity map is denser but has many obvious errors. (f) The block matching algorithm included in OpenCV is also not very dense AND contains many errors. It is only provided as a reference of what can be easily obtained with a freely available quasi-real-time block matching algorithm. (e) Proposed method ACBM+SS.}
\label{fig:shrub}
\end{figure}

\begin{figure}[h]
\begin{center}
\begin{minipage}{.45\columnwidth}
\centering
 \includegraphics[height=2.5cm]{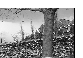}\\
	(a)\\ \vspace{0.1cm}
\includegraphics[height=2.5cm]{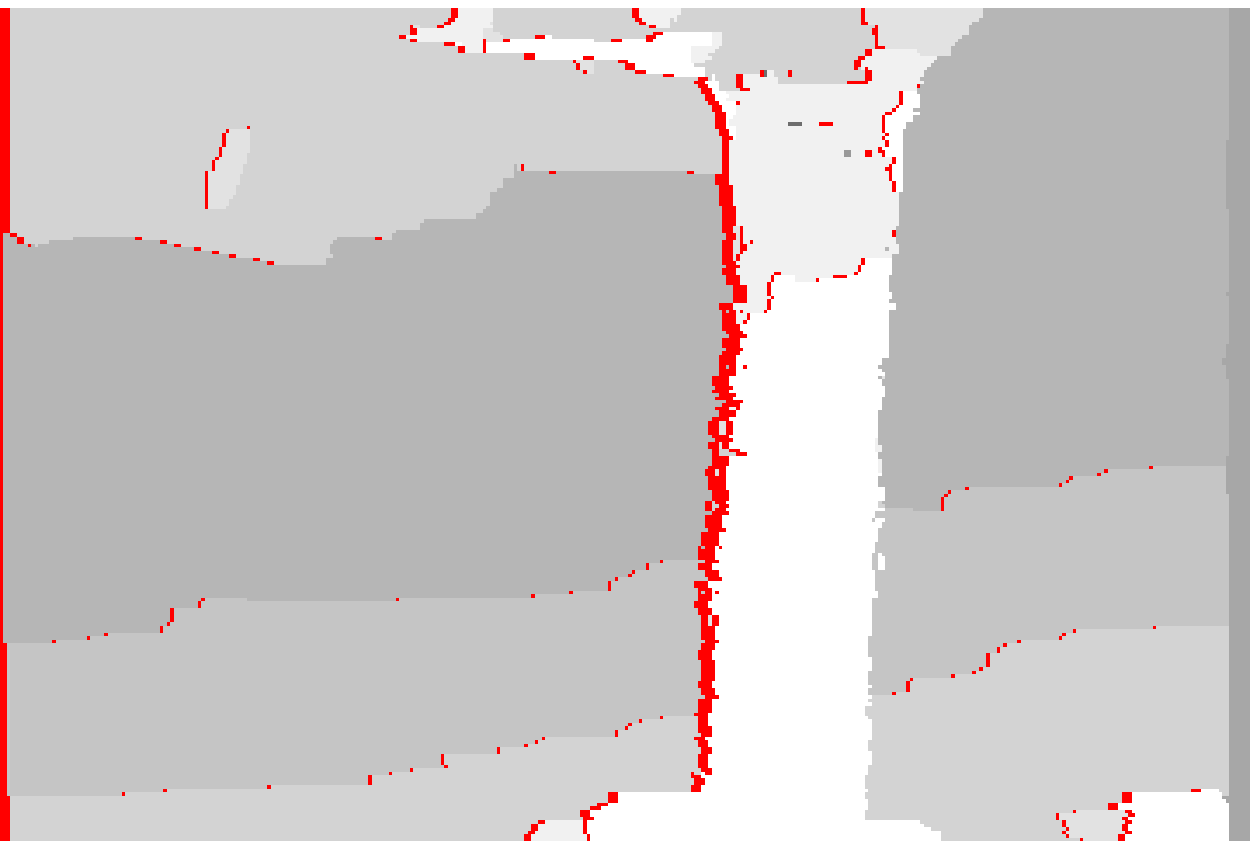}\\
	(c)\\
\end{minipage}
\hfil
\begin{minipage}{.45\columnwidth}
\centering
\includegraphics[height=2.5cm]{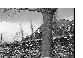}\\
	(b)\\ \vspace{0.2cm}
\includegraphics[height=2.45cm]{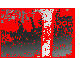}\\
	(d)\\
\end{minipage}
\end{center}
\caption{Flower-garden scene. (a) and (b) Reference and secondary images. (c) Graph Cuts (method of \cite{Kolmogorov05}). Red points are occluded points. (d) ACBM+SS. Red points are rejected points. Density: 59\%.  Most rejected points are obviously mismatched by the graph cut algorithm, which equates the depths of trees, sky and house.}
\label{fig:flower_garden}
\end{figure}

\section{Conclusion}\label{Conclusions}
Wrong match thresholds were, in our opinion, the principal drawbacks for block-matching algorithms in stereovision.
The {\it a contrario} block-matching threshold, that  was the principal object of the present paper, combined with the self-similarity threshold is able to detect mismatches systematically, by an algorithm which is essentially parameter-free. Indeed, the only user parameter is the expected number of false matches, which can be fixed once and for all in most applications. The method indiscriminately  detects occlusions, moving objects and poor or periodic textured regions.

Mismatches in block-matching have led to the overall dominance of global energy methods. However, global methods have no validation procedure, and the proposed {\it a contrario} method must be viewed as a validation procedure, no matter what the stereo matching process was. Block-matching, together with the reliability thresholds established in this paper,  gives  a fairly dense set of reliable matches (from 50$\%$ to 80$\%$ usually). It may be objected that the obtained disparity map is not  dense.

This objection is not crucial for two reasons.
First, having only validated matches opens the path to benchmarks based on accuracy, and to raise challenges about which precision can be ultimately attained (on {\it validated} matches only). Second, knowing which matches are reliable allows one to complete a given disparity map by fusing several stereo pairs. Since disposing of multiple observations of the same scene by several cameras and/or at several different times is by now a common setting, it becomes more and more important to be able to fuse 3D information obtained from many stereo pairs. Having almost only reliable matches  in each pair promises an easy fusion. A straightforward solution in our case would be the following: Given $m>2$ images, the disparity map between each possible pair of images is computed with ACBM+SS. Then the final disparity map is the accumulated disparity map considering all meaningful matches computed with all the image pairs whenever all the computed disparities for the same pixel are coherent.

\section{Acknowledgements}
The authors thank Pascal Getreuer for helpful comments on this work.
Work partially supported by the following projects FREEDOM (ANR07-JCJC-0048-01), Callisto (ANR-09-CORD-003), ECOS Sud U06E01 and STIC Amsud (11STIC-01 - MMVPSCV).

\bibliographystyle{plain}
\bibliography{biblio_article}

\end{document}